\documentclass{article}

\usepackage[utf8]{inputenc}
\usepackage{textgreek}

\usepackage{microtype}
\usepackage{graphicx}
\usepackage{subcaption}
\usepackage{booktabs}
\usepackage{multirow}

\usepackage{hyperref}

\usepackage[preprint]{icml2026}

\usepackage{amsmath}
\usepackage{amssymb}
\usepackage{mathtools}
\usepackage{amsthm}
\usepackage{colonequals}

\usepackage[capitalize,noabbrev]{cleveref}

\theoremstyle{plain}
\newtheorem{theorem}{Theorem}[section]
\newtheorem{proposition}[theorem]{Proposition}
\newtheorem{lemma}[theorem]{Lemma}

\theoremstyle{definition}

\theoremstyle{remark}

\DeclareMathOperator*{\argmax}{arg\,max}

\DeclareMathOperator*{\Argmax}{Arg\,max}
\DeclareMathOperator*{\conv}{conv}

\newcommand{\Tau}{\mathcal{T}}

\newcommand{\eg}{\textit{e.g.}}
\newcommand{\ie}{\textit{i.e.}}

\icmltitlerunning{Tensorion: A Tensor-Aware Generalization of the Muon Optimizer}

\begin{document}

\twocolumn[
  \icmltitle{Tensorion: A Tensor-Aware Generalization of the Muon Optimizer}

  \icmlsetsymbol{equal}{*}

  \begin{icmlauthorlist}
    \icmlauthor{Vladimir Bogachev}{hse}
    \icmlauthor{Vladimir Aletov}{hse,brain}
    \icmlauthor{Alexander Molozhavenko}{hse}
    \icmlauthor{Sergei Kudriashov}{hse}
    \icmlauthor{Maxim Rakhuba}{hse}
  \end{icmlauthorlist}

  \icmlaffiliation{hse}{HSE University}
  \icmlaffiliation{brain}{Basic Research of Artificial 
  Intelligence Laboratory (BRAIn Lab)}

  \icmlcorrespondingauthor{Vladimir Bogachev}{vabogachev@hse.ru}

  \icmlkeywords{Machine Learning, ICML}

  \vskip 0.3in
]

\printAffiliationsAndNotice{} 

\begin{abstract}
  Common first-order optimizers, such as Adam, implicitly treat each parameter block as an unstructured vector, which disregards the multilinear weight structure present in many modern machine learning models. Recent work has shown that exploiting matrix structure can improve optimization dynamics. A notable example is Muon, which performs steepest descent under the spectral norm constraint. We take the next step and introduce Tensorion\footnote{https://github.com/MTML-LAB/Tensorion}, a tensor-aware optimizer that extends Muon’s constrained optimization perspective from matrices to higher-order tensors. Tensorion is built around a linear minimization oracle (LMO) over a tensor norm ball. The norm is carefully chosen to balance two objectives: tightly bounding the tensor spectral norm, while still keeping the LMO tractable. This LMO becomes computable because it reduces to operations on adaptively selected unfolding matrices. Notably, when restricted to order‑2 tensors (i.e., matrices), Tensorion recovers Muon exactly. Experiments on tensor-based computer vision problems suggest that Tensorion can offer improved convergence behavior and more stable gradient updates compared with Adam-based and existing tensor-aware baselines in the evaluated settings.

\end{abstract}

\section{Introduction}\label{sec:Introduction} 

The field of matrix optimization is emerging as a promising paradigm for the development of next-generation, efficient training algorithms for deep neural networks. In contrast to well-established methods like Adam~\citep{Kingma2014AdamAM}, which treat weights as a flat collection of numbers (parameter vectors), matrix-based approaches preserve and leverage the natural two-dimensional structure of the weight matrices.
Among these, Muon~\citep{jordan2024muon} has proven to be a particularly efficient method, achieving record-breaking performance on smaller models such as NanoGPT, and has recently been shown to scale to larger architectures~\citep{liu2025muon, kimiteam2026kimik2openagentic}.

However, many modern machine learning models are inherently tensor-valued. Flattening tensors into matrices destroys the multilinear structure and mixes independent modes, which omits information crucial for optimization. Convolutional kernels, Multi-Head Attention (MHA) \citep{devlin2019bert}, Mixture-of-Experts layers (MoE) \citep{oldfield2024multilinearmoe}, and other architectural components allow for a natural characterization as multilinear operators, \ie, higher-order tensors. 

In this work, we aim to extend the success of Muon to tensor-valued parameters. This setting naturally encompasses weights in neural networks with a convolution-based architecture and also arises in tensorization, where lower-order objects (such as vectors) are reshaped into higher-dimensional tensors \citep{novikov2015tensorizing}.

To illustrate our approach and its connection to prior work, we consider the Linear Minimization Oracle (LMO) framework~\citep{bernstein2025deriving}. Given an ascent direction, for instance the gradient $\nabla \mathcal{L}(W)$ of the loss function $\mathcal{L}(\cdot)$, one selects the search direction by solving
 \begin{equation}
    \max_{X \colon \|X\| \leq 1 }\langle \nabla \mathcal{L}\,(W), X \rangle.
    \label{eq:muon_lmo_minimization}
  \end{equation}
The specific norm  $\|\cdot\|$ chosen here essentially defines the geometry of the problem, and leads to completely different optimization trajectories \citep{pethick2025training}.
If we take the Euclidean (Frobenius) norm, the solution recovers the standard gradient descent. This is unsurprising, as the Frobenius norm treats a matrix as a flat vector and ignores its inherent structure.
In contrast, Muon employs the matrix spectral norm, which preserves the matrix structure and leads to an (approximate) orthogonalization of the gradient $\nabla \mathcal{L}(W)$ at each step. Other methods like signSGD~\citep{bernstein2018signsgd} over $\ell_\infty$-norm balls also admit this type of interpretation.

A natural step in extending this approach is to employ a norm specifically designed for tensors. 
However, unlike matrices, higher-order tensors exhibit fundamentally different properties, which significantly complicates such a generalization. 
In particular, one cannot straightforwardly use the well-known tensor generalization of the spectral norm. 

Our contributions are as follows:
\begin{enumerate}
    \item We show how to implicitly define a tensor norm that, on the one hand, serves as a tight upper bound to the spectral tensor norm (see Sec.~\ref{subsec:rel_ten_spec}), on the other hand,  ensures a tractable solution to the LMO problem (see Sec.~\ref{subsec:rel_lmo_sol}). In a nutshell, using our norm results in adaptively choosing unfolding matrices of a tensor depending on their nuclear matrix norms. We call the resulting optimizer ``Tensorion''.
    \item Based on theoretical results for randomly generated tensors (see Sec.~\ref{subsec:mot_choice_unf}) and an extensive evaluation of different unfolding strategies (see Sec.~\ref{subsec:ablation-tau}), we develop the best practical strategy for choosing the unfolding approach without requiring additional recomputation at each iteration.
    \item For convolutional layers, we show (see Sec.~\ref{subsec:tens_for_conv}) how the proposed method relates to the LMO problem for the actual Jacobian of the convolution operation (similarly to how $W$ is a Jacobian of $y = Wx + b$ for fully-connected layers).
    \item We evaluate the proposed optimizer on multiple CV classification datasets across CNN and transformer-based architectures (see Sec.~\ref{sec:experiments}) and observe consistent improvements over conventional SGD and Adam-based optimizers.
\end{enumerate}

\section{Related work}\label{sec:related_work}

Tensorion is related to matrix-aware and tensor-aware optimization methods. Muon introduced matrix orthogonalization into optimization~\citep{jordan2024muon}, with later work improving its training recipe ~\citep{liu2025muon}, connecting it to spectral constraints~\citep{chen2025muon}, and extending it through manifold constraints~\citep{bernstein2025manifolds}, fixed-rank fine-tuning~\citep{bogachev2026lora}, distributed orthonormalized updates~\citep{ahn2025dion}, and general norm-based LMOs~\citep{riabinin2025gluon}. A parallel line of structured optimizers uses matrix or tensor structure for preconditioning: K-FAC uses Kronecker-factored curvature~\citep{martens2015optimizing}, with extensions to convolutional layers~\citep{grosse2016kronecker} and transformers~\citep{eschenhagen2023kronecker}; Shampoo maintains mode-wise tensor preconditioners~\citep{gupta2018shampoo}, with further analysis in~\citep{morwani2024new}; SOAP combines Shampoo eigenspaces with Adam-like updates~\citep{vyas2024soap}; and~\cite{bernstein2024old} give an operator-norm view. In contrast to these preconditioning methods, Tensorion does not maintain auxiliary matrices and instead acts directly through tensor unfoldings.

Tensor methods have also been used to compress or structure neural network parameters, including tensor-structured filters~\citep{rigamonti2013learning}, Tensor Train layers~\citep{novikov2015tensorizing}, stable tensor decompositions~\citep{phan2020stable}, and Tucker-style parameterizations~\citep{molozhavenkosfett,peshekhonov2024training}. Related tensor structures have been applied to multi-head attention~\citep{gu2025tensorllm,li2026lestd}, mixture-of-experts models~\citep{yuebintd}, fixed-rank tensor optimization~\citep{mo2025parameter}. These approaches typically modify the model parameterization, whereas Tensorion preserves the original parameters and uses tensor structure only in the optimizer. Finally, unlike~\citep{zhang2026teontensorizedorthonormalizationlayerwise}, which stacks layerwise gradients into auxiliary tensors, Tensorion operates directly on each layer's tensor parameters.

\section{Preliminaries}\label{sec:Preliminaries}

\subsection{Norms, dual norms, and LMOs}

The LMO in Eq.~\eqref{eq:muon_lmo_minimization} is naturally described through dual norms.  
For a norm $\|\cdot\|$ on a finite-dimensional inner-product space, its dual norm is
\begin{equation}
    \|M\|^\dagger
    \colonequals
    \max_{\|X\|\leq 1} \langle M, X\rangle .
\end{equation}
Consequently, the maximizers of the linear oracle over the unit ball are exactly the subgradients of the dual norm:
\begin{equation}
    \begin{split}
    &\operatorname*{Argmax}_{\|X\|\leq 1}\langle M,X\rangle
     = 
    \partial \|M\|^\dagger
    =\\
    &\left\{
    Y \;\middle|\;
    \|Y\|\leq 1,\;
    \langle M,Y\rangle=\|M\|^\dagger
    \right\},
    \end{split}
    \label{eq:muon_lmo_dual_connection}
\end{equation}
a standard consequence of convex duality~\citep{rockafellar1997convex}.  
The corresponding minimization oracle is obtained by changing the sign.

For Muon, the relevant primal norm is the matrix spectral norm,
$
    \|M\|_2 = \sigma_1(M)$,
whose dual is the nuclear norm
\begin{equation}
    \|M\|_2^\dagger
    =
    \|M\|_*
    =
    \sum_{i=1}^{r}\sigma_i(M),
\end{equation}
where $r=\operatorname{rank}(M)$ and
$\sigma_1(M)\geq \dots \geq \sigma_r(M)>0$ are the nonzero singular values.

Let $M=U\Sigma V^\top$ be a compact SVD. Then $UV^\top$ is a maximizer since
$
    \langle M,UV^\top\rangle
    =
    \operatorname{Tr}(\Sigma)
    =
    \|M\|_* $.
Thus, $UV^\top \in \partial \|M\|_*$, while $-UV^\top$ solves the corresponding minimization oracle. This dual-norm viewpoint is the basic mechanism behind the Muon update and will be the starting point for our tensor extension. For completeness, for a convex function $f:V\to\mathbb{R}$, the subdifferential at $x$ is
\begin{equation}
\begin{split}
    \partial &f(x)
    \colonequals \\
    &\left\{
    g\in V
    \;\middle|\;
    f(y)\geq f(x)+\langle g,y-x\rangle,\ \forall y\in V
    \right\}.
\end{split}
\end{equation}

\subsection{The tensor case}

Let $X,Y\in\mathbb{R}^{n_1\times\dots\times n_d}$ be $d$-th order tensors.  
We use the Euclidean tensor inner product
\begin{equation}
    \langle X,Y\rangle
    \colonequals
    \sum_{i_1=1}^{n_1}
    \dots
    \sum_{i_d=1}^{n_d}
    X_{i_1,\dots,i_d}Y_{i_1,\dots,i_d},
    \label{eq:tensor_euclidean_inner_prod}
\end{equation}
and the associated Frobenius norm $\|X\|_F \colonequals \sqrt{\langle X,X\rangle} $. The matrix spectral norm admits the variational form:
\begin{equation}
    \|X\|_2
    =
    \max_{\|u\|_2=\|v\|_2=1}
    |u^\top Xv|
    =
    \max_{\|u\|_2=\|v\|_2=1}
    |\langle X,uv^\top\rangle|.
\end{equation}
Then its multilinear analogue is a direct generalization of a matrix case:
\begin{equation}
    \|X\|_\sigma
    \colonequals
    \sup_{\substack{
    u^{(k)}\in\mathbb{R}^{n_k},\ \|u^{(k)}\|_2=1\\
    k=1,\dots,d
    }}
    \left|
    \left\langle
    X,\,
    u^{(1)}\otimes\dots\otimes u^{(d)}
    \right\rangle
    \right|,
    \label{eq:tensor_spectral_norm}
\end{equation}
where
$\big(u^{(1)}\otimes\dots\otimes u^{(d)}\big)_{i_1,\dots,i_d}
=
u^{(1)}_{i_1}\cdots u^{(d)}_{i_d}$.
For $d=2$, this definition reduces to the usual matrix spectral norm. For $d\geq 3$, however, the tensor spectral norm and its dual are substantially harder to compute and analyze~\citep{friedland2018nuclear}. This motivates the tractable relaxation developed in Section~\ref{sec:Tensorion}.

We will frequently use tensor unfoldings.  
For $\tau\subset\{1,\dots,d\}$, let $\tau^c$ denote its complement and define
\begin{equation}
    N_\tau \colonequals \prod_{k\in\tau} n_k,
    \qquad
    N_{\tau^c} \colonequals \prod_{k\notin\tau} n_k,
\end{equation}
with the convention that an empty product equals one.  
The $\tau$-unfolding of $X$, denoted by $X_{(\tau)}$, is the matrix obtained by grouping the modes in $\tau$ into rows and the remaining modes into columns:
\begin{equation}
    \big(X_{(\tau)}\big)_{\mathbf{j}_\tau,\mathbf{j}_{\tau^c}}
    \colonequals
    X_{j_1,\dots,j_d},
    \qquad
    X_{(\tau)}
    \in
    \mathbb{R}^{N_\tau\times N_{\tau^c}},
    \label{eq:tau_unfolding}
\end{equation}
where $\mathbf{j}_\tau$ and $\mathbf{j}_{\tau^c}$ are the corresponding lexicographically ordered multi-indices. 
We denote by $\mathrm{fold}_\tau(\cdot)$ the reverse operation, which maps a $\tau$-unfolded matrix back to the corresponding tensor of shape $n_1\times\cdots\times n_d$.
The standard mode-$k$ matricization is the special case
\begin{equation}
    X_{(k)} \colonequals X_{(\{k\})}.
    \label{eq:matricization}
\end{equation}

\section{Tensorion}\label{sec:Tensorion} 

This section introduces \emph{Tensorion}, our tensor-valued extension of Muon. The central difficulty is that the natural LMO over the tensor spectral norm ball is computationally intractable. Tensorion addresses this obstacle by replacing the tensor spectral norm with a duality-based relaxation that is both tighter than any unfolding spectral norm and computable. 

\subsection{Relaxed tensor spectral norm}\label{subsec:rel_ten_spec}

The tensor spectral norm in Eq.~\eqref{eq:tensor_spectral_norm} is the natural analogue of the matrix spectral norm and is therefore the most direct candidate for extending Muon beyond matrices. However, computing the tensor spectral norm, and hence solving the corresponding exact LMO, is NP-hard in general~\citep{hillar2013most}. This makes the direct tensor-spectral formulation impractical for modern tensor-valued neural-network parameters.

A standard relaxation is obtained by upper-bounding the tensor spectral norm with the spectral norm of a matrix unfolding. Namely, for any $\tau \subset \{1,\dots,d\}$, it is well known~\citep{wang2017operator} that
\begin{equation}
    \|X\|_\sigma \leq \|X_{(\tau)}\|_2 .
\end{equation}
This observation motivates unfolding-based heuristics, including the empirical strategy used in~\cite{jordan2024muon} for convolutional layers. However, selecting a single unfolding hard-codes a particular tensor geometry. Different layers, and even different tensor shapes within the same architecture, may favor different groupings of modes. Thus, the choice of unfolding becomes an additional design decision rather than a property of the optimization problem itself.

A natural attempt to remove this ambiguity is to aggregate several candidate unfoldings. Given a family of mode subsets $\Tau$, for example $\Tau=\{\{1\},\{1,2\}\}$ or $\Tau = 2^{\{1, 2, \ldots, d\}}$, one may consider
\begin{equation}
    \min_{\tau\in\Tau}\|X_{(\tau)}\|_2
    \quad\text{or}\quad
    \max_{\tau\in\Tau}\|X_{(\tau)}\|_2 .
\end{equation}
The first construction is generally not a norm: the minimum of norms need not satisfy the triangle inequality. The second construction is a norm, but it is often too conservative. In particular, it can collapse to the Frobenius norm in the presence of small or singleton modes, which are common in convolutional kernels.

To see this, consider the extreme case $X\in\mathbb{R}^{1\times n_2\times\dots\times n_d}$ and suppose that $\{1\}\in\Tau$. Then
\begin{equation}
    \max_{\tau\in\Tau}\|X_{(\tau)}\|_2
    =
    \|X_{(\{1\})}\|_2
    =
    \|\operatorname{vec}(X)^\top\|_2
    =
    \|X\|_F ,
\end{equation}
where $\operatorname{vec}(\cdot)$ denotes column-major vectorization. Hence, the relaxation becomes the Frobenius norm, which may substantially overestimate the tensor spectral norm and lose sensitivity to the multilinear structure of $X$. Moreover, this behavior prevents the relaxation from recovering the matrix Muon update on tensors of size $1\times m\times n$, which should effectively behave as matrices.

Tensorion avoids this pathology by moving the relaxation to the dual side. Instead of aggregating primal unfolding spectral norms, we construct a lower bound on the dual tensor spectral norm using unfolding nuclear norms. Specifically, we define:
\begin{equation}
  \|X\|_\Sigma^\dagger
  \colonequals
  \max_{\tau\in\Tau}\|X_{(\tau)}\|_* .
  \label{eq:tensorion_norm}
\end{equation}
The relaxed tensor spectral norm is then defined by duality:
\begin{equation}
    \|X\|_\Sigma
    \colonequals
    \left(\|X\|_\Sigma^\dagger\right)^\dagger .
\end{equation}
This dual construction is the key relaxation underlying Tensorion. It preserves the desirable upper-bound property required for a spectral-norm relaxation, while avoiding the overly conservative behavior of the maximum over primal unfolding norms. Supplementary Material~\ref{sec:appendix_relaxation_tightness} provides an empirical analysis of relaxation tightness along optimization trajectories.

The following proposition formalizes the relationship between the proposed norm, the tensor spectral norm, and unfolding-based spectral relaxations. In particular, Tensorion always upper-bounds the tensor spectral norm, yet it is no looser than the best spectral norm among the selected unfoldings.

\begin{proposition}
\label{proposition:tensorion_and_spectral_norm_ineqs}
Let $X\in\mathbb{R}^{n_1\times\dots\times n_d}$ be any tensor. Then
\begin{equation}
    \|X\|_\Sigma^\dagger \leq \|X\|_\sigma^\dagger ,
\end{equation}
and consequently
\begin{equation}
    \|X\|_\sigma
    \leq
    \|X\|_\Sigma
    \leq
    \min_{\tau\in\Tau}\|X_{(\tau)}\|_2 .
    \label{eq:tensorion_dual_bounds_nuclear}
\end{equation}
\end{proposition}
    Proof is in Supplementary Materials, Section~\ref{sec:appendix_primal_norm},
Proposition~\ref{proposition:tensorion_and_spectral_norm_ineqs_proof}.

\subsection{Relaxed LMO solution}\label{subsec:rel_lmo_sol}

Similar in spirit to the Muon optimizer, our method solves an LMO over the $\|\cdot\|_\Sigma$ at each iteration. 
Given a tensor $M$ (\eg, the gradient or a momentum estimate at some iteration), we compute
\begin{equation}
  X^{\text{opt}} \in \Argmax_{X \colon \|X\|_\Sigma \le 1} 
  \langle M, X \rangle.
  \label{eq:tensorion_norm_LMO}
\end{equation}

Proposition~\ref{proposition:tensorion_norm_LMO_solution} characterizes the solution set of Eq.~\eqref{eq:tensorion_norm_LMO}.

\begin{proposition}\label{proposition:tensorion_norm_LMO_solution}
  The solution set of the LMO in Eq.~\eqref{eq:tensorion_norm_LMO} is the subdifferential of the Tensorion dual norm at $M$:
  \begin{equation}
    X^{\text{opt}}\in \partial \|M\|_\Sigma^\dagger = \conv_{\tau \in I}\,\left( \mathrm{fold}_\tau \left( \partial \|M_{(\tau)} \|_* \right) \right), 
    \label{eq:subdiff_dual_norm_sol}
  \end{equation}
  where  $\conv{(\cdot)}$ denotes a convex hull and
  \begin{equation}
    I = \left\{ \tau \,\middle|\, 
    \|M_{(\tau)}\|_* = \|M\|_\Sigma^\dagger \right\} \subseteq 2^{\{1, \dotsc, d\}}
    \label{eq:maximizing_indices}
  \end{equation}
  is a nonempty set of unfolding indices attaining the maximal nuclear norm.
\end{proposition}
  Proof is in Supplementary Materials Section~\ref{sec:appendix_primal_norm}, Proposition~\ref{proposition:tensorion_norm_LMO_solution_proof}.
  
Importantly,  Proposition~\ref{proposition:tensorion_norm_LMO_solution} gives us a practical way to solve the LMO in Eq.~\eqref{eq:tensorion_norm_LMO} by finding an unfolding with the largest nuclear norm.
In particular, for any $\tau \in I$, one admissible choice is obtained by folding $UV^\top$ back along the unfolding $\tau$, where $U$ and $V$ are the left and right singular vectors from the singular value decomposition $M_{(\tau)} = U\Sigma V^\top$.

The best unfolding can be different for different layers and even change during the optimization process.
We found that searching for optimal unfoldings during every iteration of the optimization process is excessive (see numerical experiments). 
Even though it is possible to do it efficiently, using, \eg, Newton-Schulz iteration run in parallel, there are simpler ways that lead to similar optimization trajectory in practice.

\subsection{Motivation for the choice of unfolding shape}\label{subsec:mot_choice_unf}

Observe that for any unfolding $M_{(\tau)} \in \mathbb{R}^{m_{\tau} \times n_{\tau}}$
\begin{equation}
  \|M\|_F^2 
  = \|M_{(\tau)}\|_F^2 
  = \sum_{j=1}^{r_\tau} \left(\sigma_j^{(\tau)}\right)^2,\; r_{\tau} = \mathrm{rank}\,(M_{(\tau)}),
  \label{eq:frobenius_norm_unfoldless}
\end{equation}
where $\{\sigma_j^{(\tau)}\}$ denote the singular values of $M_{(\tau)}$, $m_\tau = \prod_{i \in \tau} n_i, \, n_\tau=\prod_{i \in \tau^c} n_i$.
Since unfolding only permutes tensor entries, the Frobenius norm of the unfolding is independent of the choice of $\tau$. Hence, for a fixed Frobenius norm, maximizing the nuclear norm reduces to maximizing the sum of singular values subject to a fixed sum of their squares. By the Cauchy--Schwarz inequality, the maximum is attained when the norm budget is distributed uniformly across all available singular values. Consequently, the maximal nuclear norm grows with the number of potentially nonzero singular values, motivating the heuristic of choosing an unfolding with the largest possible number of nonzero singular values.

The following formal statement further supports the heuristic by providing a bound that justifies choosing the most balanced unfolding.

\begin{proposition}
\label{prop:randomized_bound}
Let the entries of $ M\in\mathbb R^{n_1\times\cdots\times n_d}$ be i.i.d. copies of a mean-zero random variable $Z$ 
with $\|Z\|_{\psi_1}\le K$, and let $N=\prod_{i=1}^{d} n_i$. 
Consider any unfolding $M_{(\tau)} \in \mathbb{R}^{m_\tau \times n_\tau}$ with $m_\tau \leq n_\tau$ 
such that $m_{\tau}n_{\tau}  = \prod_{i=1}^d n_i$. 
Then, there exists a constant $C > 0$ depending only on $K$ such that, 
for any $\delta\in(0,1)$, with probability at least $1-\delta$, 
simultaneously for every unfolding $\tau$,
\begin{equation}
\label{eq:prob_bound}
\|M_{(\tau)}\|_*\le C \Bigl(m_\tau \sqrt{n_\tau} + \sqrt{m_\tau} \Bigl(\sqrt{\log\tfrac2\delta} + \log \tfrac{2}{\delta}\Bigr)\Bigr) 
\end{equation}
Moreover, for a fixed total dimension $m_{\tau}n_{\tau}$, the value of this upper bound is maximized when the unfolding is square, i.e., $m_\tau = n_\tau = \sqrt{m_\tau n_\tau}$ if such unfolding exists, or with the most "balanced" one otherwise.
\end{proposition}

    Proof is in Supplementary Materials Section~\ref{sec:proof_randomized_bound}, Proposition~\ref{prop:randomized_bound_app}.

Therefore, given that all the assumptions hold, we expect for the gradient tensor unfoldings whose row and column dimensions are more balanced to yield larger nuclear norms. 
Motivated by this observation, we consider the following as our primary strategy for the static choice of unfolding prior to the optimization loop:
\begin{equation}
    \tau^{\text{opt}} \colonequals \argmax\limits_{\tau \in \Tau} \left(\min\left\{m_{\tau}, n_{\tau}\right\}\right),
    \label{eq:tau_opt_unfolding}
\end{equation}
and restrict the orthogonalization step to the chosen $\tau^{\text{opt}}$-unfolding.

\subsection{Tensorion algorithm}
The proposed theoretical framework leads to our main contribution, the
\emph{Tensorion Optimizer}. Algorithm~\ref{alg:one_step_of_tensorion} presents
the version of the method used in the experimental section.

The lines $3$ to $5$ contain the momentum initialization and the computation of
the $\tau^{\mathrm{opt}}$-unfolding from Eq.~\eqref{eq:tau_opt_unfolding}, which is
used to obtain a faster LMO solution. Within the main optimization loop,
Tensorion follows the Muon update scheme, with the key modification appearing in
line~$9$. 

First, the function $\mathrm{unfold}(\cdot,\tau)$ maps a $d$-dimensional tensor
to its $\tau$-unfolding, precalculated in line~$4$. Then, Newton--Schulz orthogonalization, denoted
$\mathrm{NS}(\cdot)$, is applied. Finally, $\mathrm{fold}(\cdot,\tau)$ maps the
result back to the original tensor shape. Changes with respect to the Muon
implementation of \cite{liu2025muon} are highlighted in blue.

\begin{algorithm}[tb]
    \caption{Tensorion}
    \label{alg:one_step_of_tensorion}
\begin{algorithmic}[1]
  \STATE {\bfseries Input:}
    Weight tensor $W \in \mathbb{R}^{n_1 \times \dotsm \times n_d}$, 
    step sizes $\{\eta_k\}$,
    momentum coefficient $\beta$.
  \STATE {\bfseries Output:} Weight tensor $W' \in \mathbb{R}^{n_1 \times \dotsm \times n_d}$.

  \STATE $M_0 = 0$;
  \STATE \textcolor{blue}{$\textstyle{\tau \colonequals
  \argmax_{\tau}
  \big(
  \min\big\{
  \prod_{i \in \tau} n_i,\,
  \prod_{i \not\in \tau} n_i
  \big\}
  \big)}$};

  \STATE \textcolor{blue}{$\textstyle{m_{\tau}, n_{\tau}
  \colonequals
  \prod_{i \in \tau} n_i,\,
  \prod_{i \not\in \tau} n_i}$};

  \FOR{$k \colonequals 1, \dotsc$}
    \STATE Compute gradient $G_k \in \mathbb{R}^{n_1 \times \dotsm \times n_d}$;
    \STATE $M_k \colonequals G_k + \beta M_{k - 1}$;
    \STATE $X_k \colonequals
    \textcolor{blue}{\mathrm{fold}}
    \left(
    \mathrm{NS}
    \left(
    \textcolor{blue}{\mathrm{unfold}}(M_k,\tau)
    \right),
    \tau
    \right)$;
    \STATE $\hat{\eta}_k \colonequals
    \eta_k \cdot 0.2
    \sqrt{
    \max(
    \textcolor{blue}{m_{\tau}},
    \textcolor{blue}{n_{\tau}}
    )
    }$;
    \STATE $W_{k + 1} \colonequals
    (1 - \gamma \eta_k) W_k - \hat{\eta}_k X_k$;
  \ENDFOR

  \STATE {\bfseries return} $W_{k + 1}$.
\end{algorithmic}
\end{algorithm}

\subsection{Tensorion for Convolution Layers}\label{subsec:tens_for_conv}

In this subsection, let us discuss what it means to apply Tensorion to a convolutional kernel tensor.
Let $X \in \mathbb{R}^{\text{c}_{\text{in}} \times n \times n}$ be the input to a convolutional layer and 
let $K \in \mathbb{R}^{\text{c}_{\text{out}} \times \text{c}_{\text{in}} \times h \times w}$ be its kernel. 
Let $Y \in \mathbb{R}^{\text{c}_{\text{out}} \times n \times n}$ denote the output tensor. 
The convolution operation can be viewed as a linear mapping
\begin{equation}
    \mathrm{vec}(Y) = T_K \, \mathrm{vec}(X), 
    \quad 
    T_K \in \mathbb{R}^{\text{c}_{\text{out}} n^2 \times \text{c}_{\text{in}} n^2}
\end{equation}
where $\mathrm{vec}(\cdot)$ denotes the vectorization operator and $T_K$ is a matrix of the layer determined by the kernel $K$.
Note that $T_K$ is a large matrix with additional structure, \ie, it is a sparse block matrix where each block represents a two-level Toeplitz structure~\citep{grishinatight}. It is used solely for theoretical analysis purposes and is never formed explicitly. 
The spectral properties of this convolution operator have been shown to affect generalization and training stability \citep{singla2019fantastic, agarwal2019improving}. According to \cite{grishinatight}, the spectral norm 
of the linear map $T_K$ is bounded by the spectral tensor norm as stated in Theorem~\ref{theorem:grishinatight}.

\begin{theorem}[\citet{grishinatight}]\label{theorem:grishinatight}
Let $T_K \in \mathbb{R}^{\text{c}_{\text{out}} n^2 \times \text{c}_{\text{in}} n^2}$ be the Jacobian matrix of a convolutional layer with stride $s=1$ and zero padding with parameter $p \ge 0$, or circular padding. 
Let $K \in \mathbb{R}^{\text{c}_{\text{out}} \times \text{c}_{\text{in}} \times h \times w}$ be the convolution kernel. Then
\begin{equation}
    \|K\|_\sigma \le \|T_K\|_2 \le \sqrt{hw}\,\|K\|_\sigma .
\end{equation}
\end{theorem}

Let
$\mathcal{A}\colon K \mapsto T_K$ denote the linear map that associates a
convolution kernel with its corresponding convolution matrix. We define a
differentiable function
\[
\ell \colon
\mathbb{R}^{c_{\mathrm{out}}n^2 \times c_{\mathrm{in}}n^2}
\to \mathbb{R}
\]
such that the loss $\mathcal{L}$ of the convolution kernel can be expressed
in terms of its corresponding convolution matrix as
\begin{equation}
    \mathcal{L}\colon
    \mathbb{R}^{c_{\mathrm{out}}\times c_{\mathrm{in}}\times h\times w}
    \to \mathbb{R},
    \qquad
    \mathcal{L}(K)
    = \ell(T_K)
    = \ell(\mathcal{A}K).
\end{equation}
By the chain rule,
\begin{equation}
    \nabla_K \mathcal{L}(K)
    =
    \mathcal{A}^\ast \nabla_{T_K}\ell(T_K),
\end{equation}

where $\nabla_{T_K}\ell(T_K)$ denotes the gradient of $\ell$ with respect to $T_K$, projected onto the subspace of convolution matrices $\operatorname{Im}(\mathcal{A})$. Consequently, $\nabla_{T_K}\ell(T_K)$ has the same Toeplitz structure as $T_K$.

Hence, for any kernel perturbation
$H \in
\mathbb{R}^{\mathrm{c}_{\mathrm{out}}
\times \mathrm{c}_{\mathrm{in}}
\times h \times w}$,
\begin{equation*}
    \left\langle
        \nabla_{T_K}\ell(T_K), T_H
    \right\rangle
    =
    \left\langle
        \nabla_{T_K}\ell(T_K), \mathcal{A}H
    \right\rangle
    =
    \left\langle
        \nabla_K\mathcal{L}(K), H
    \right\rangle.
\end{equation*}
Therefore, applying the Muon LMO to the convolution operator while
restricting the update to preserve the convolutional structure is exactly
equivalent to the following LMO:
\begin{equation}
\label{eq:structured_conv_lmo}
\begin{split}
    &\max_{\substack{
        T_H \in \operatorname{Im}(\mathcal{A})\\
        \|T_H\|_2 \le 1
    }}
    \left\langle
        \nabla_{T_K}\ell(T_K), T_H
    \right\rangle =
    \\
    &\max_{\|T_H\|_2 \le 1}
    \left\langle
        \nabla_K\mathcal{L}(K), H
    \right\rangle =
    \\
    &\max_{\|\mathcal{A}H\|_2 \le 1}
    \left\langle
        \nabla_K\mathcal{L}(K), H
    \right\rangle .
\end{split}
\end{equation}

Since $T_H=\mathcal{A}H$ is the convolution matrix associated with the
kernel perturbation $H$, Theorem~\ref{theorem:grishinatight} gives
\begin{equation}
    \|H\|_\sigma
    \le
    \|T_H\|_2
    \le
    \sqrt{hw}\,\|H\|_\sigma.
\end{equation}
Consequently, denoting
\begin{equation*}
    \mathbb{B}_{\sigma}(r)
    := \{H : \|H\|_{\sigma} \le r\},
    \quad
    \mathbb{B}_{T}(r)
    := \{H : \|T_H\|_2 \le r\},
\end{equation*}
Theorem~\ref{theorem:grishinatight} yields the following inclusions between the corresponding
feasible sets:
\begin{equation}
\label{eq:conv_tensor_ball_inclusions}
    \mathbb{B}_{\sigma}\!\left(\frac{1}{\sqrt{hw}}\right)
    \subseteq
    \mathbb{B}_{T}(1)
    \subseteq
    \mathbb{B}_{\sigma}(1).
\end{equation}
Thus, the LMO over the unit tensor spectral norm ball $\mathbb{B}_{\sigma}(1)$ is an outer
relaxation of the structured Muon LMO in
Eq.~\eqref{eq:structured_conv_lmo}, while the tensor spectral norm ball $\mathbb{B}_{\sigma}\!\left({1}/{\sqrt{hw}}\right)$ provides an inner approximation. Since the LMOs
over these two tensor spectral norm balls differ only by a constant
rescaling of the feasible set, their optimal update directions coincide
up to scale. This establishes a direct connection between Muon applied
to the actual linear convolution operator and an LMO formulated directly
over the convolution kernel using the tensor spectral norm.

Combining Proposition~\ref{proposition:tensorion_and_spectral_norm_ineqs}
with Eq.~\eqref{eq:conv_tensor_ball_inclusions}, we obtain
\begin{equation*}
    \left\{
        H : \|H\|_{\Sigma} \leq \frac{1}{\sqrt{hw}}
    \right\}
    \subseteq
    \mathbb{B}_{\sigma}\!\left(\frac{1}{\sqrt{hw}}\right)
    \subseteq
    \mathbb{B}_{T}(1).
\end{equation*}
Therefore, the Tensorion LMO in Eq.~\eqref{eq:tensorion_norm_LMO},
with its feasible set scaled by $1/\sqrt{hw}$ and the maximizing unfolding
selected online according to Eq.~\eqref{eq:maximizing_indices}, provides
a tractable inner approximation to the structured Muon LMO in
Eq.~\eqref{eq:structured_conv_lmo}.

\section{Experiments}\label{sec:experiments}

To validate the derived heuristic, we first ablate the choice of tensor unfolding index set $\tau$ and study its effect on the optimization dynamics (Section~\ref{subsec:ablation-tau}). We then evaluate whether the resulting improvements translate into higher final accuracy on standard CNN (Section~\ref{subsec:tiny_imagenet}) and Vision Transformer (ViT) classification benchmarks (Section~\ref{subsec:vit_swin}). Additional details, including hyperparameters (Section~\ref{sec:appendix_hyperparameters}) and further experiments on the Laplacian eigenproblem (Section~\ref{subsec:laplacian-spectrum}) and CIFAR datasets (Section~\ref{appc:cifar}), are provided in the Supplementary Material.

\subsection{Ablation on an Unfolding Set}
\label{subsec:ablation-tau}

We isolate the impact of the unfolding index set $\tau$ and evaluate our heuristic choice $\tau^{\mathrm{opt}}$ from Eq.~\eqref{eq:tau_opt_unfolding}, referred to as the \emph{offline} strategy. As a brute-force reference, we also compare against an \emph{online} strategy that evaluates all admissible unfoldings at each iteration and selects the one maximizing Eq.~\eqref{eq:maximizing_indices}. While potentially more adaptive, this online strategy is substantially more expensive in practice.

For this ablation, we evaluate all seven nontrivial unfoldings of a 4D convolutional kernel up to the complement symmetry $\tau \sim \tau^c$. Static baselines use one unfolding across all layers, while the offline strategy selects $\tau^{\mathrm{opt}}$ per layer. To isolate unfolding selection from approximation effects, we use exact SVDs instead of NS at line 7 in Algorithm~\ref{alg:one_step_of_tensorion} and train ResNet-18 on CIFAR-10 for 25 epochs with identical settings and independently tuned constant learning rates; see Supplementary Materials, Section~\ref{subsec:appendix_tau_ablation}. Figure~\ref{fig:tau-ablation} and Table~\ref{tab:tau-ablation-epoch25} show that the choice of unfolding has a substantial effect on performance. The tested unfoldings separate into two accuracy regimes: the lower-performing group consists of $\tau=\{3\}$, $\tau=\{4\}$, and $\tau=\{1,2\}$, whereas the remaining choices form a higher-performing group. The differences within each group are not statistically significant under this protocol, and therefore we do not attempt to rank individual unfoldings inside the high-accuracy group.

The online strategy, which explicitly maximizes Eq.~\eqref{eq:maximizing_indices} at every iteration, achieves the highest mean accuracy, as expected from its additional adaptivity. However, this comes at a substantially higher computational cost, and the observed gap to the other high-accuracy methods is not statistically resolved in this ablation. In contrast, the offline heuristic $\tau^{\mathrm{opt}}$ is inexpensive and still falls into the high-accuracy regime. \emph{This suggests that the heuristic succeeds at avoiding unfavorable unfoldings and recovers a competitive fixed-per-layer choice without requiring brute-force online selection.} Moreover, the elapsed time for Tensorion with offline $\tau$-selection strategy is comparable to the Muon optimizer.

\begin{figure}[h!tb]
    \centering

    \begin{subfigure}[t]{\linewidth}
        \centering
        \includegraphics[width=\linewidth]{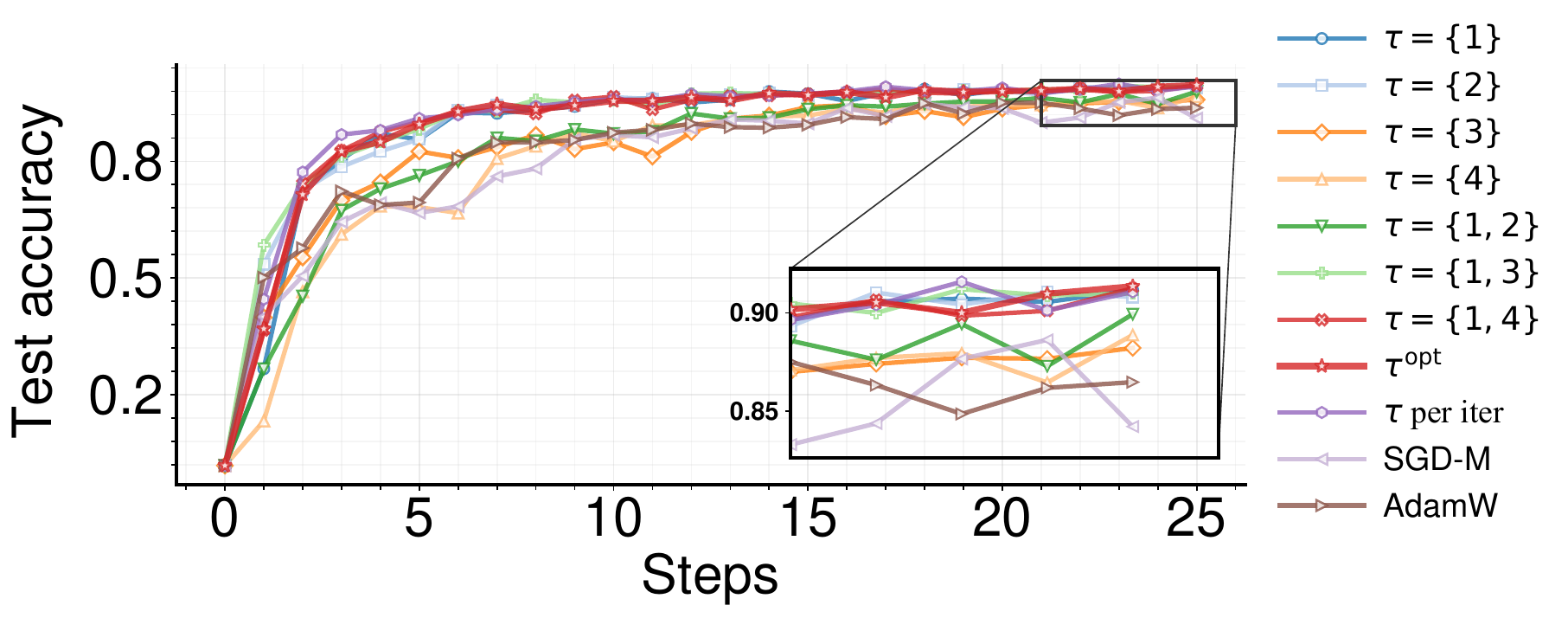}
    \end{subfigure}
    \begin{subfigure}[t]{\linewidth}
        \centering
        \includegraphics[width=\linewidth]{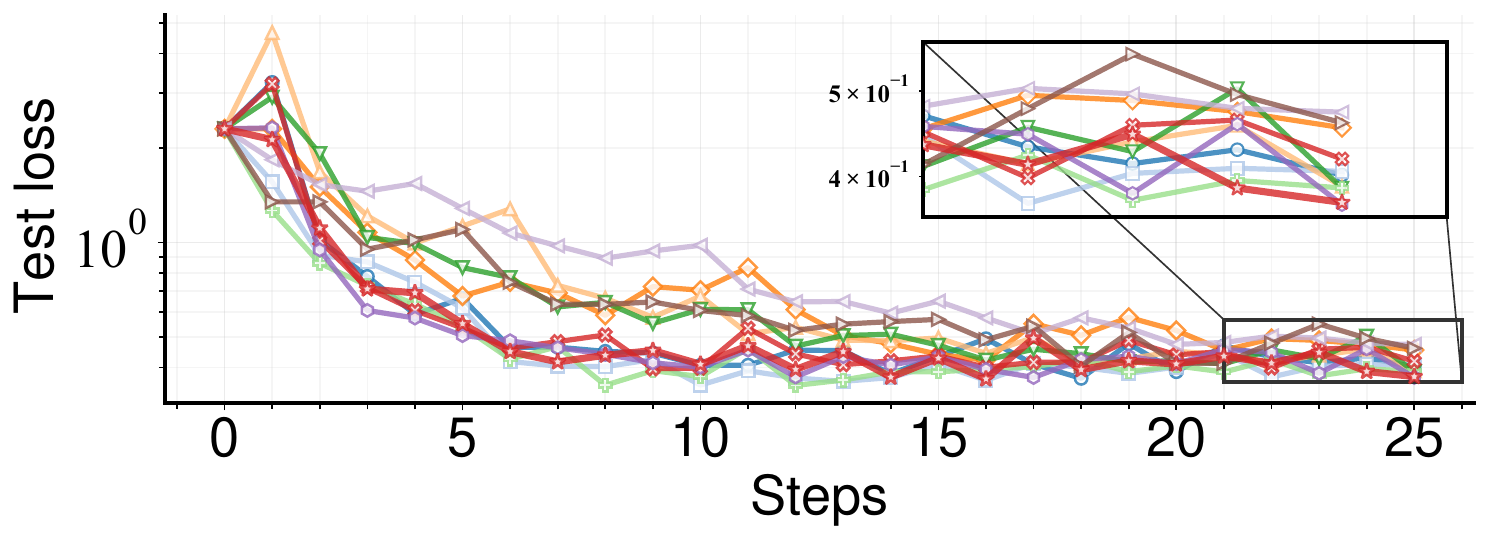}
    \end{subfigure}\hfill
    \begin{subfigure}[t]{\linewidth}
        \centering
        \includegraphics[width=\linewidth]{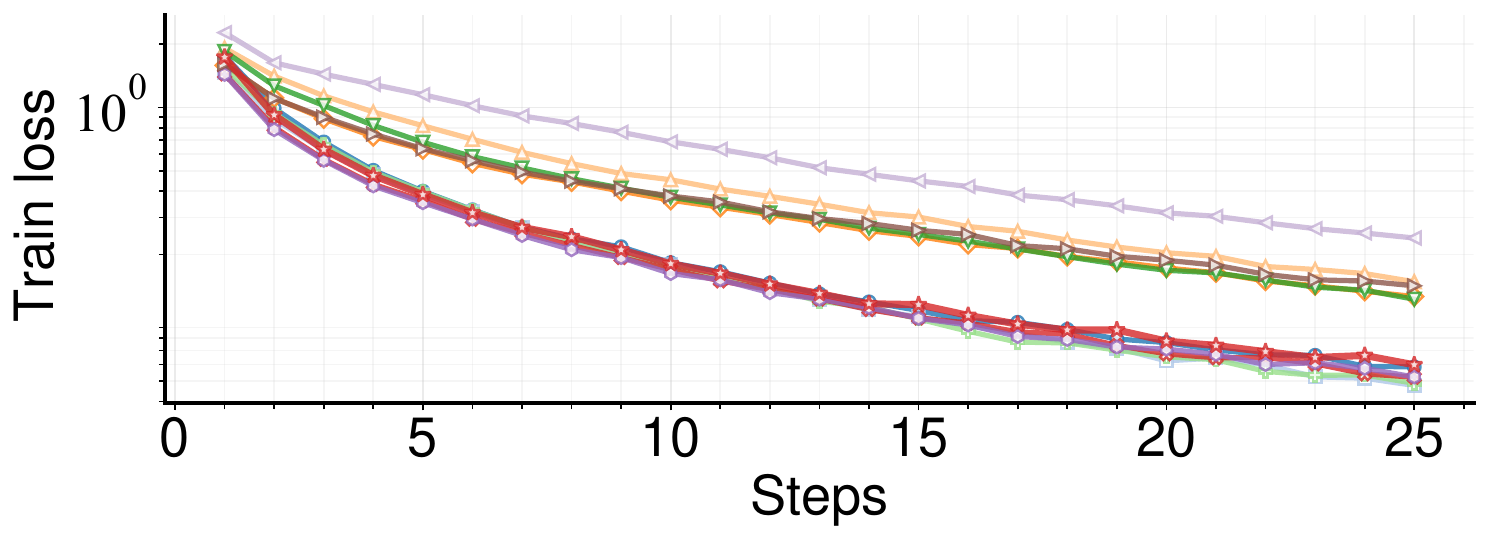}
    \end{subfigure}

    \caption{Ablation on the unfolding set \(\tau\) for ResNet-18 on CIFAR-10 (25 epochs; constant LR, no scheduler) using exact SVD for orthogonalization. Curves labeled ``\(\tau=\{\cdot\}\)'' use a fixed unfolding index set \(\tau\subset\{1,2,3,4\}\).``\(\tau^{\mathrm{opt}}\)'' selects \(\tau\) once per layer via the offline heuristic in Eq.~\eqref{eq:tau_opt_unfolding}. ``\(\tau\) per iteration'' performs online selection by explicitly maximizing \(\|X_{(\tau)}\|_*\) over \(\tau\) at each iteration as in Eq.~\eqref{eq:maximizing_indices}. SGD-M: SGD+Momentum.}
    \label{fig:tau-ablation}
\end{figure}

We include the fixed unfolding $\tau=\{1\}$ as the Muon comparison, since this choice recovers the original Muon update in our framework. The main conclusion of the ablation is not a fine-grained ordering among the competitive methods, but rather that Tensorion is sensitive to the unfolding choice and that the proposed heuristic provides a practical way to avoid poor unfoldings without per-iteration search. SGD-M and AdamW~\citep{loshchilov2017decoupled} are included as baselines and achieve lower test accuracy under the same constant-learning-rate protocol. To evaluate the applicability of Tensorion in a larger-scale image classification setting, we next provide a set of experiments on various computer vision problems.
\begin{table}[h!tp]
\centering
\caption{Ablation on the unfolding set $\tau$ for ResNet-18 on CIFAR-10 after 25 epochs. Reported values are mean $\pm$ standard deviation over 10 runs; SGD-M: SGD+Momentum.}
\label{tab:tau-ablation-epoch25}
\setlength{\tabcolsep}{4pt}
\renewcommand{\arraystretch}{1.25}
\small
\begin{tabular}{lcc}
\toprule
\textbf{Method} 
& \textbf{Test acc. $\uparrow$} 
& \textbf{Test loss $\downarrow$} \\
\midrule
$\tau=\{1\}$ 
& $90.67_{\pm\,0.32}$ 
& $0.424_{\pm\,0.026}$ \\

$\tau=\{2\}$ 
& $90.62_{\pm\,0.44}$ 
& $\mathbf{0.404_{\pm\,0.024}}$ \\

$\tau=\{3\}$ 
& $87.40_{\pm\,1.00}$ 
& $0.485_{\pm\,0.051}$ \\

$\tau=\{4\}$ 
& $88.05_{\pm\,0.69}$ 
& $0.423_{\pm\,0.028}$ \\

$\tau=\{1,2\}$ 
& $88.83_{\pm\,0.39}$ 
& $0.430_{\pm\,0.022}$ \\

$\tau=\{1,3\}$ 
& $90.30_{\pm\,0.51}$ 
& $0.426_{\pm\,0.031}$ \\

$\tau=\{2,4\}$ 
& $90.77_{\pm\,0.45}$ 
& $0.432_{\pm\,0.038}$ \\

$\tau^{\mathrm{opt}}$ 
& $90.44_{\pm\,0.36}$ 
& $0.441_{\pm\,0.024}$ \\

online $\tau$ 
& $\mathbf{91.10_{\pm\,0.39}}$ 
& $0.416_{\pm\,0.026}$ \\

SGD-M 
& $86.10_{\pm\,1.82}$ 
& $0.527_{\pm\,0.095}$ \\

AdamW 
& $86.04_{\pm\,1.05}$ 
& $0.507_{\pm\,0.041}$ \\
\bottomrule
\end{tabular}
\end{table}

\subsection{Tiny ImageNet with ResNet50}\label{subsec:tiny_imagenet}

We train a ResNet-50 model on Tiny ImageNet \citep{tiny-imagenet} for 200 epochs, comparing Tensorion against SGD and AdamW baselines. Tensorion is applied only to tensor-valued parameters of order $d \geq 3$ (\ie, convolutional kernels), while all remaining parameters ($d \leq 2$) are optimized with SGD. We report mean $\pm$ std over multiple seeds.

\begin{figure}[h]
    \centering
    \begin{subfigure}{\linewidth}
        \centering
        \hspace{-0.5cm} \includegraphics[width=\linewidth]{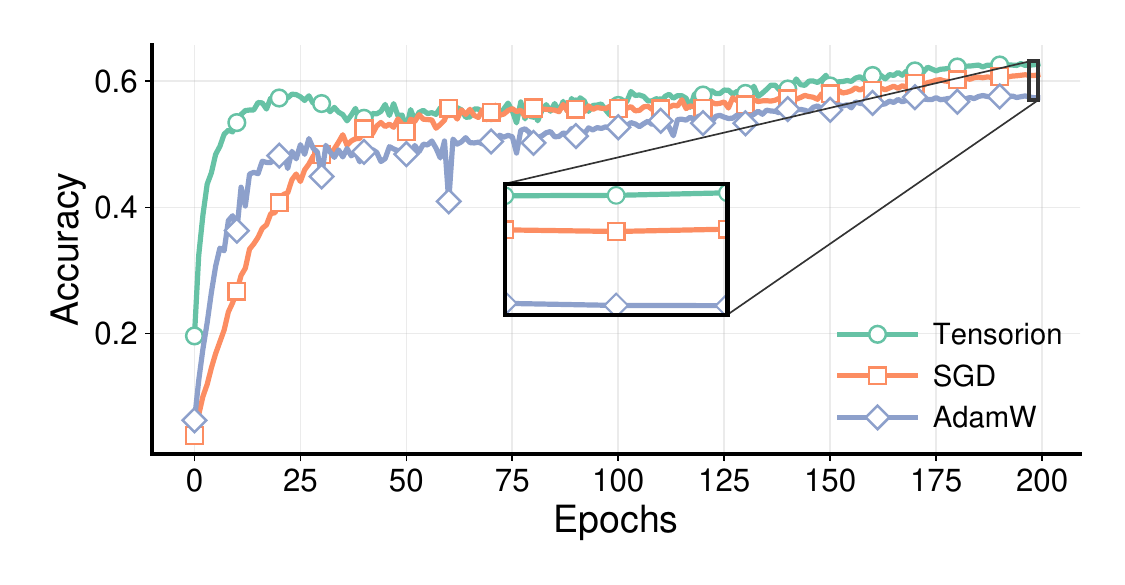}
        \caption{Test accuracy vs Epoch}
        \label{fig:timagenet-acc}
    \end{subfigure} %
    \begin{subfigure}{\linewidth}
        \centering
        \includegraphics[width=\linewidth]{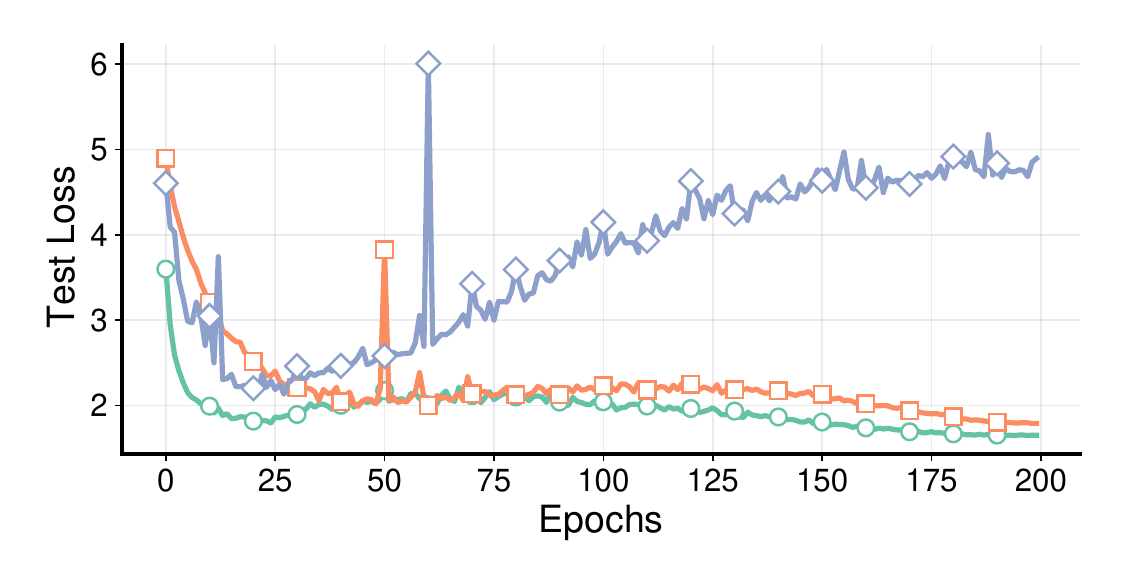}
        \caption{Test loss vs Epoch}
        \label{fig:timagenet-test-loss}
    \end{subfigure}

    \caption{Tiny ImageNet results for ResNet-50 trained for $200$ epochs. We report mean \(\pm\) std over multiple seeds for test accuracy 
    and test loss 
    Tensorion is used only for tensor-valued blocks with order \(d\ge 3\) (\eg, convolutional kernels), while all remaining parameters (\(d\le 2\)) are optimized with a first-order method (SGD).}
    \label{fig:timagenet}
\end{figure}

\begin{table}[h!tp]
\centering
\caption{Main results. Final test loss and test accuracy. C10: CIFAR-10; C100: CIFAR-100; Tiny-IN: Tiny ImageNet; RN: ResNet; SGD-M: SGD+Momentum; T-SGD: Tensorion+SGD; T-AdamW: Tensorion+AdamW.}
\label{tab:main}
\footnotesize
\setlength{\tabcolsep}{3pt}
\renewcommand{\arraystretch}{0.95}
\begin{tabular}{llccc}
\toprule
\textbf{Dataset} 
& \textbf{Model} 
& \textbf{Opt.} 
& \textbf{Loss} $\downarrow$ 
& \textbf{Acc.} $\uparrow$ \\
\midrule

\multirow{4}{*}{C10}
& \multirow{4}{*}{RN-18}
& SGD-M 
& $0.25_{\pm0.02}$ 
& $94.1_{\pm0.13}$ \\
& & AdamW 
& $0.44_{\pm0.02}$ 
& $93.7_{\pm0.3}$ \\
& & T-SGD 
& $\mathbf{0.21_{\pm0.01}}$ 
& $\mathbf{95.3_{\pm0.1}}$ \\
& & T-AdamW 
& $0.59_{\pm0.08}$ 
& $94.34_{\pm0.3}$ \\

\midrule

\multirow{4}{*}{C100}
& \multirow{4}{*}{RN-34}
& SGD-M 
& $\mathbf{0.89_{\pm0.13}}$ 
& $78.2_{\pm0.2}$ \\
& & AdamW 
& $1.33_{\pm0.1}$ 
& $74.2_{\pm0.1}$ \\
& & T-SGD 
& $0.98_{\pm0.09}$ 
& $\mathbf{79.1_{\pm0.1}}$ \\
& & T-AdamW 
& $4.89_{\pm0.11}$ 
& $73.8_{\pm0.2}$ \\

\midrule

\multirow{3}{*}{Tiny-IN}
& \multirow{3}{*}{RN-50}
& SGD-M 
& $1.79_{\pm0.06}$ 
& $61.01_{\pm0.02}$ \\
& & AdamW 
& $\mathbf{1.33_{\pm0.05}}$ 
& $57.4_{\pm0.1}$ \\
& & T-SGD 
& $1.65_{\pm0.06}$ 
& $\mathbf{62.7_{\pm0.1}}$ \\

\bottomrule
\end{tabular}
\end{table}

As shown in Figure~\ref{fig:timagenet}, Tensorion consistently outperforms both baselines in test accuracy throughout training. The advantage is even more pronounced in terms of test loss: AdamW exhibits clear overfitting, with its test loss increasing after roughly 50 steps, whereas Tensorion achieves the lowest and monotonically decreasing loss. These findings are consistent with our ResNet-18/-34 experiments on CIFAR-10/-100 reported in Appendix~\ref{appc:cifar} and summarized alongside all main results in Table~\ref{tab:main}.

\begin{figure}[h!tp]
    \centering
    \begin{subfigure}{0.9\linewidth}
        \centering        \includegraphics[width=\linewidth]{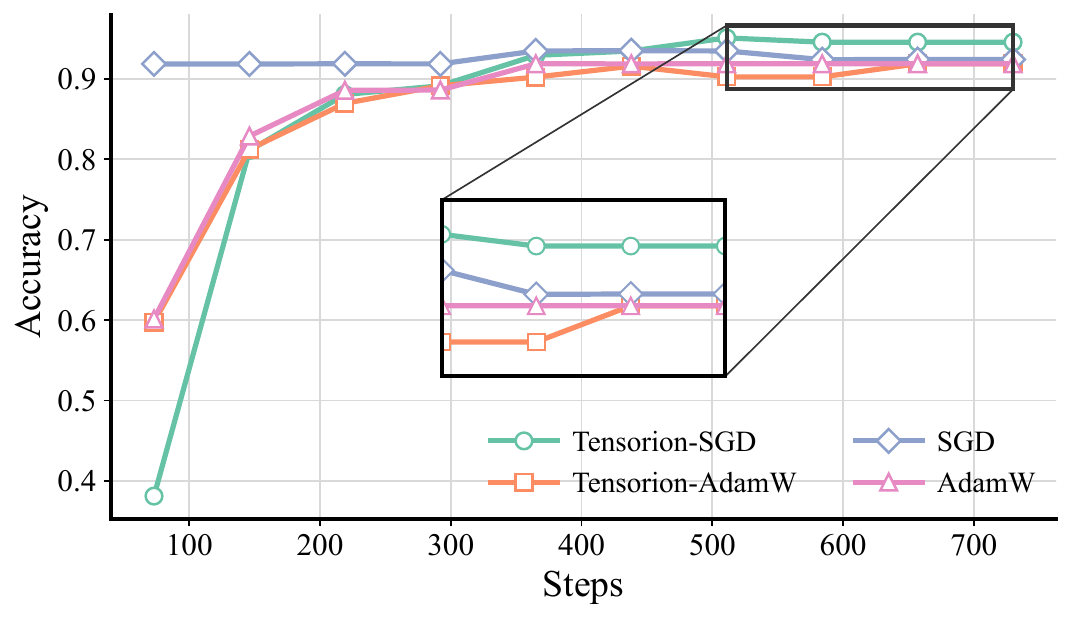}
        \caption{Swin Base}
        \label{fig:swin_acc}
    \end{subfigure}
    \begin{subfigure}{0.9\linewidth}
        \centering
        \includegraphics[width=\linewidth]{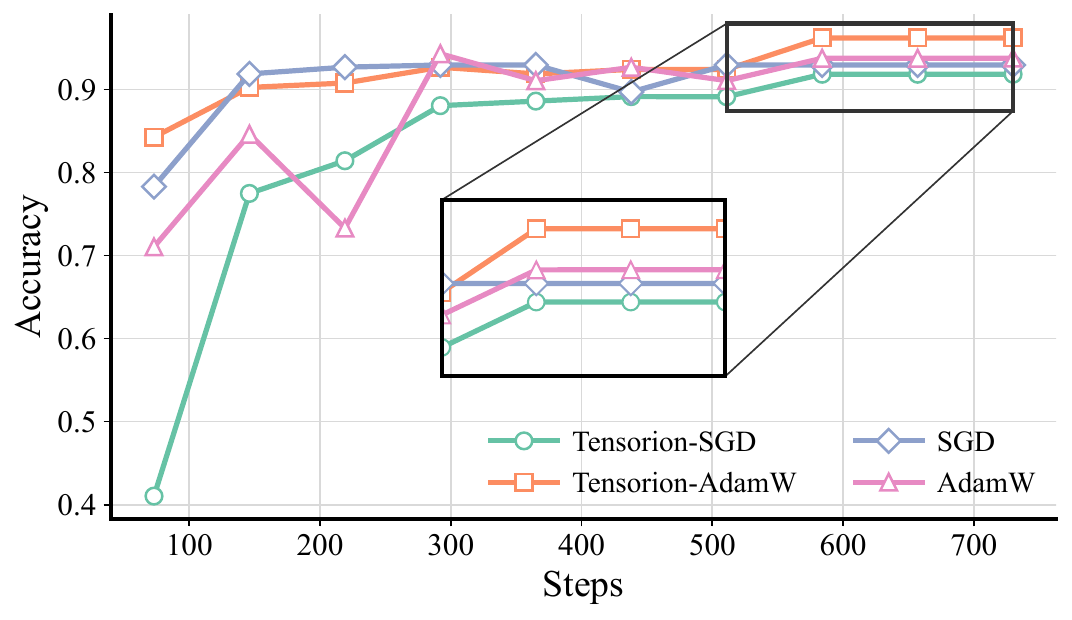}
        \caption{ViT Base}
        \label{fig:vit_acc}
    \end{subfigure}

    \caption{Imagenette results for Swin and ViT models trained for 5 epochs. Tensorion is used only for tensor-valued blocks with order \(d\ge 3\) (\eg, convolutional kernels), while all remaining parameters (\(d\le 2\)) are optimized with either SGD or AdamW.}
    \label{fig:Imagenette2}
\end{figure}

\subsection{Tensorion consistently improves performance of Vision Transformers.}\label{subsec:vit_swin}

Although Vision Transformers mostly consist of common matrix layers, the convolutional patch-embedding module is crucial for optimization stability and model generalization~\citep{xiao2021early}. We aim to show that applying Tensorion to this module alone is sufficient to improve performance across multiple scales and architectures. Specifically, we initialize ViT~\citep{dosovitskiy2020image} (Small, Base, Medium) and Swin Transformer~\citep{liu2021swin} (Tiny, Small, Base) models from pretrained checkpoints in the \texttt{timm} library, fine-tune them on the Imagenette dataset~\citep{Howard_Imagenette_2019} for 5 epochs with a fixed learning-rate scheduler and weight decay across optimizers, and evaluate performance at the end of each epoch. Tensorion is applied only to the convolutional patch-embedding layer. All remaining parameters are optimized with AdamW or SGD, respectively. Optimal learning rates are tuned via grid search for each setup. We report the best performance across algorithms in Fig.~\ref{fig:Imagenette2}. Detailed results are reported in Section~\ref{sec:details_vit_swin}.

\section{Conclusion}\label{sec:conclusion}
We propose Tensorion, a novel optimizer that generalizes Muon to tensors via a norm that lower bounds the spectral norm of any tensor unfolding, leading to a practical relaxation of the tensor spectral‑norm LMO problem. A heuristic offline method for choosing the unfolding achieves performance comparable to the online optimal choice. Experiments across multiple CV tasks and architectures show consistent improvements over conventional optimizers.

\section*{Limitations}\label{sec:limitations}
Our study is limited primarily to computer vision models. We do not evaluate video models or architectures operating on sequential visual domains. We also do not assess whether similar behavior appears for convolutional or any other tensor-valued layers in other domains beyond vision. Finally, our experiments are limited in scale: while the observed trends are encouraging, we cannot determine from the present evidence whether the method scales as reliably as Muon or other established optimizers in substantially larger training regimes.

\section*{Acknowledgements}\label{sec:acknowledgements}
The work was supported by the grant for research centers in the field of AI provided by the Ministry of Economic Development of the Russian Federation in accordance with the agreement 000000C313925P4E0002 and the agreement with HSE University № 139-15-2025-009.
This research was supported in part through computational resources of HPC facilities at HSE University \cite{kostenetskiy2021hpc}.

\bibliographystyle{icml2026}
\bibliography{tensorion}

@string(ECCV  = {Eur. Conf. Comput. Vis.})

@string(NeurIPS = {Adv. Neural Inform. Process. Syst.})

@string(ICIP  = {IEEE Int. Conf. Image Process.})

@string(ECCV  = {ECCV})

@string(NeurIPS = {NeurIPS})

@string(ICIP  = {ICIP})

@article{dosovitskiy2020image,
	title        = {An image is worth 16x16 words: Transformers for image recognition at scale},
	author       = {Dosovitskiy, Alexey and Beyer, Lucas and Kolesnikov, Alexander and Weissenborn, Dirk and Zhai, Xiaohua and Unterthiner, Thomas and Dehghani, Mostafa and Minderer, Matthias and Heigold, Georg and Gelly, Sylvain and others},
	year         = 2020,
	journal      = {arXiv preprint arXiv:2010.11929}
}

@book{girsanov2012lectures,
	title        = {Lectures on mathematical theory of extremum problems},
	author       = {Girsanov, Igor Vladimirovich},
	year         = 2012,
	publisher    = {Springer Science \& Business Media}
}

@article{limblind,
	title        = {Blind Multilinear Identification},
	author       = {Lim, Lek-Heng and Comon, Pierre},
	year         = 2014,
	journal      = {IEEE Transactions on Information Theory},
	volume       = 60,
	number       = 2,
	pages        = {1260--1280},
	doi          = {10.1109/TIT.2013.2291876}
}

@article{riabinin2025gluon,
	title        = {Gluon: Making muon \& scion great again!(bridging theory and practice of lmo-based optimizers for llms)},
	author       = {Riabinin, Artem and Shulgin, Egor and Gruntkowska, Kaja and Richt{\'a}rik, Peter},
	year         = 2025,
	journal      = {arXiv preprint arXiv:2505.13416}
}

@inproceedings{rigamonti2013learning,
	title        = {Learning separable filters},
	author       = {Rigamonti, Roberto and Sironi, Amos and Lepetit, Vincent and Fua, Pascal},
	year         = 2013,
	booktitle    = {Proceedings of the IEEE conference on computer vision and pattern recognition},
	pages        = {2754--2761}
}

@article{ahn2025dion,
	title        = {Dion: Distributed orthonormalized updates},
	author       = {Ahn, Kwangjun and Xu, Byron and Abreu, Natalie and Fan, Ying and Magakyan, Gagik and Sharma, Pratyusha and Zhan, Zheng and Langford, John},
	year         = 2025,
	journal      = {arXiv preprint arXiv:2504.05295}
}

@article{chen2025muon,
	title        = {Muon optimizes under spectral norm constraints},
	author       = {Chen, Lizhang and Li, Jonathan and Liu, Qiang},
	year         = 2025,
	journal      = {arXiv preprint arXiv:2506.15054}
}

@inproceedings{phan2020stable,
	title        = {Stable low-rank tensor decomposition for compression of convolutional neural network},
	author       = {Phan, Anh-Huy and Sobolev, Konstantin and Sozykin, Konstantin and Ermilov, Dmitry and Gusak, Julia and Tichavsk{\`y}, Petr and Glukhov, Valeriy and Oseledets, Ivan and Cichocki, Andrzej},
	year         = 2020,
	booktitle    = {European Conference on Computer Vision},
	pages        = {522--539},
	organization = {Springer}
}

@article{bernstein2025manifolds,
	title        = {Modular Manifolds},
	author       = {Jeremy Bernstein},
	year         = 2025,
	journal      = {Thinking Machines Lab: Connectionism},
	doi          = {10.64434/tml.20250926},
	note         = {https://thinkingmachines.ai/blog/modular-manifolds/}
}

@inproceedings{bernstein2018signsgd,
	title        = {signSGD: Compressed optimisation for non-convex problems},
	author       = {Bernstein, Jeremy and Wang, Yu-Xiang and Azizzadenesheli, Kamyar and Anandkumar, Animashree},
	year         = 2018,
	booktitle    = {International conference on machine learning},
	pages        = {560--569},
	organization = {PMLR}
}

@article{novikov2015tensorizing,
	title        = {Tensorizing neural networks},
	author       = {Novikov, Alexander and Podoprikhin, Dmitrii and Osokin, Anton and Vetrov, Dmitry P},
	year         = 2015,
	journal      = {Advances in neural information processing systems},
	volume       = 28,
	doi          = {10.5555/2969239.2969289}
}

@article{wang2017operator,
	title        = {Operator norm inequalities between tensor unfoldings on the partition lattice},
	author       = {Wang, Miaoyan and Duc, Khanh Dao and Fischer, Jonathan and Song, Yun S},
	year         = 2017,
	journal      = {Linear algebra and its applications},
	publisher    = {Elsevier},
	volume       = 520,
	pages        = {44--66},
	doi          = {10.1016/j.laa.2017.01.017}
}

@article{singla2019fantastic,
	title        = {Fantastic four: Differentiable bounds on singular values of convolution layers},
	author       = {Singla, Sahil and Feizi, Soheil},
	year         = 2019,
	journal      = {arXiv preprint arXiv:1911.10258}
}

@inproceedings{agarwal2019improving,
	title        = {Improving robustness to adversarial examples by encouraging discriminative features},
	author       = {Agarwal, Chirag and Nguyen, Anh and Schonfeld, Dan},
	year         = 2019,
	booktitle    = {2019 IEEE international conference on image processing (ICIP)},
	pages        = {3801--3805},
	doi          = {10.1109/ICIP.2019.8803601},
	organization = {IEEE}
}

@inproceedings{grishinatight,
	title        = {Tight and Efficient Upper Bound on Spectral Norm of Convolutional Layers},
	author       = {Grishina, Ekaterina and Gorbunov, Mikhail and Rakhuba, Maxim},
	year         = 2025,
	booktitle    = {Computer Vision -- ECCV 2024},
	publisher    = {Springer Nature Switzerland},
	address      = {Cham},
	pages        = {19--34},
	doi          = {10.1007/978-3-031-73024-5_2},
	isbn         = {978-3-031-73024-5},
	editor       = {Leonardis, Ale{\v{s}} and Ricci, Elisa and Roth, Stefan and Russakovsky, Olga and Sattler, Torsten and Varol, G{\"u}l}
}

@article{liu2025muon,
	title        = {Muon is scalable for llm training},
	author       = {Liu, Jingyuan and Su, Jianlin and Yao, Xingcheng and Jiang, Zhejun and Lai, Guokun and Du, Yulun and Qin, Yidao and Xu, Weixin and Lu, Enzhe and Yan, Junjie and others},
	year         = 2025,
	journal      = {arXiv preprint arXiv:2502.16982}
}

@article{friedland2018nuclear,
	title        = {Nuclear norm of higher-order tensors},
	author       = {Friedland, Shmuel and Lim, Lek-Heng},
	year         = 2018,
	journal      = {Mathematics of Computation},
	volume       = 87,
	number       = 311,
	pages        = {1255--1281},
	doi          = {10.1090/mcom/3239}
}

@book{rockafellar1997convex,
	title        = {Convex analysis},
	author       = {Rockafellar, R Tyrrell},
	year         = 1997,
	publisher    = {Princeton university press},
	volume       = 28
}

@misc{kimiteam2026kimik2openagentic,
	title        = {Kimi K2: Open Agentic Intelligence},
	author       = {Kimi Team and Yifan Bai and Yiping Bao and Y. Charles and Cheng Chen and Guanduo Chen and Haiting Chen and Huarong Chen and Jiahao Chen and Ningxin Chen and Ruijue Chen and Yanru Chen and Yuankun Chen and Yutian Chen and Zhuofu Chen and Jialei Cui and Hao Ding and Mengnan Dong and Angang Du and Chenzhuang Du and Dikang Du and Yulun Du and Yu Fan and Yichen Feng and Kelin Fu and Bofei Gao and Chenxiao Gao and Hongcheng Gao and Peizhong Gao and Tong Gao and Yuyao Ge and Shangyi Geng and Qizheng Gu and Xinran Gu and Longyu Guan and Haiqing Guo and Jianhang Guo and Xiaoru Hao and Tianhong He and Weiran He and Wenyang He and Yunjia He and Chao Hong and Hao Hu and Yangyang Hu and Zhenxing Hu and Weixiao Huang and Zhiqi Huang and Zihao Huang and Tao Jiang and Zhejun Jiang and Xinyi Jin and Yongsheng Kang and Guokun Lai and Cheng Li and Fang Li and Haoyang Li and Ming Li and Wentao Li and Yang Li and Yanhao Li and Yiwei Li and Zhaowei Li and Zheming Li and Hongzhan Lin and Xiaohan Lin and Zongyu Lin and Chengyin Liu and Chenyu Liu and Hongzhang Liu and Jingyuan Liu and Junqi Liu and Liang Liu and Shaowei Liu and T. Y. Liu and Tianwei Liu and Weizhou Liu and Yangyang Liu and Yibo Liu and Yiping Liu and Yue Liu and Zhengying Liu and Enzhe Lu and Haoyu Lu and Lijun Lu and Yashuo Luo and Shengling Ma and Xinyu Ma and Yingwei Ma and Shaoguang Mao and Jie Mei and Xin Men and Yibo Miao and Siyuan Pan and Yebo Peng and Ruoyu Qin and Zeyu Qin and Bowen Qu and Zeyu Shang and Lidong Shi and Shengyuan Shi and Feifan Song and Jianlin Su and Zhengyuan Su and Lin Sui and Xinjie Sun and Flood Sung and Yunpeng Tai and Heyi Tang and Jiawen Tao and Qifeng Teng and Chaoran Tian and Chensi Wang and Dinglu Wang and Feng Wang and Hailong Wang and Haiming Wang and Jianzhou Wang and Jiaxing Wang and Jinhong Wang and Shengjie Wang and Shuyi Wang and Si Wang and Xinyuan Wang and Yao Wang and Yejie Wang and Yiqin Wang and Yuxin Wang and Yuzhi Wang and Zhaoji Wang and Zhengtao Wang and Zhengtao Wang and Zhexu Wang and Chu Wei and Qianqian Wei and Haoning Wu and Wenhao Wu and Xingzhe Wu and Yuxin Wu and Chenjun Xiao and Jin Xie and Xiaotong Xie and Weimin Xiong and Boyu Xu and Jinjing Xu and L. H. Xu and Lin Xu and Suting Xu and Weixin Xu and Xinran Xu and Yangchuan Xu and Ziyao Xu and Jing Xu and Jing Xu and Junjie Yan and Yuzi Yan and Hao Yang and Xiaofei Yang and Yi Yang and Ying Yang and Zhen Yang and Zhilin Yang and Zonghan Yang and Haotian Yao and Xingcheng Yao and Wenjie Ye and Zhuorui Ye and Bohong Yin and Longhui Yu and Enming Yuan and Hongbang Yuan and Mengjie Yuan and Siyu Yuan and Haobing Zhan and Dehao Zhang and Hao Zhang and Wanlu Zhang and Xiaobin Zhang and Yadong Zhang and Yangkun Zhang and Yichi Zhang and Yizhi Zhang and Yongting Zhang and Yu Zhang and Yutao Zhang and Yutong Zhang and Zheng Zhang and Haotian Zhao and Yikai Zhao and Zijia Zhao and Huabin Zheng and Shaojie Zheng and Longguang Zhong and Jianren Zhou and Xinyu Zhou and Zaida Zhou and Jinguo Zhu and Zhen Zhu and Weiyu Zhuang and Xinxing Zu},
	year         = 2026,
	url          = {https://arxiv.org/abs/2507.20534},
	eprint       = {2507.20534},
	archiveprefix = {arXiv},
	primaryclass = {cs.LG}
}

@misc{zhang2026teontensorizedorthonormalizationlayerwise,
	title        = {TEON: Tensorized Orthonormalization Beyond Layer-Wise Muon for Large Language Model Pre-Training},
	author       = {Ruijie Zhang and Yequan Zhao and Ziyue Liu and Zhengyang Wang and Dongyang Li and Yupeng Su and Sijia Liu and Zheng Zhang},
	year         = 2026,
	url          = {https://arxiv.org/abs/2601.23261},
	eprint       = {2601.23261},
	archiveprefix = {arXiv},
	primaryclass = {cs.LG}
}

@article{Kingma2014AdamAM,
	title        = {Adam: A Method for Stochastic Optimization},
	author       = {Diederik P. Kingma and Jimmy Ba},
	year         = 2014,
	journal      = {CoRR},
	volume       = {abs/1412.6980},
	url          = {https://doi.org/10.48550/arXiv.1412.6980}
}

@article{hillar2013most,
	title        = {Most Tensor Problems Are NP-Hard},
	author       = {Christopher J. Hillar and Lek-Heng Lim},
	year         = 2009,
	journal      = {ArXiv},
	volume       = {abs/0911.1393},
	doi          = {10.1145/2512329}
}

@misc{jordan2024muon,
	title        = {Muon: An optimizer for hidden layers in neural networks},
	author       = {Keller Jordan and Yuchen Jin and Vlado Boza and Jiacheng You and Franz Cesista and Laker Newhouse and Jeremy Bernstein},
	year         = 2024,
	url          = {https://kellerjordan.github.io/posts/muon/}
}

@misc{bernstein2025deriving,
	title        = {Deriving Muon},
	author       = {Jeremy Bernstein},
	year         = 2025,
	url          = {https://jeremybernste.in/writing/deriving-muon}
}

@misc{2023mmpretrain,
	title        = {OpenMMLab's Pre-training Toolbox and Benchmark},
	author       = {MMPreTrain Contributors},
	year         = 2023,
	howpublished = {\url{https://github.com/open-mmlab/mmpretrain}}
}

@inproceedings{oldfield2024multilinearmoe,
	title        = {Multilinear mixture of experts: scalable expert specialization through factorization},
	author       = {Oldfield, James and Georgopoulos, Markos and Chrysos, Grigorios G. and Tzelepis, Christos and Panagakis, Yannis and Nicolaou, Mihalis A. and Deng, Jiankang and Patras, Ioannis},
	year         = 2024,
	booktitle    = {Proceedings of the 38th International Conference on Neural Information Processing Systems},
	location     = {Vancouver, BC, Canada},
	publisher    = {Curran Associates Inc.},
	address      = {Red Hook, NY, USA},
	series       = {NIPS '24},
	isbn         = 9798331314385,
	articleno    = 1680,
	numpages     = 42
}

@article{Krizhevsky09learningmultiple,
	title        = {Learning multiple layers of features from tiny images},
	author       = {Krizhevsky, Alex and Hinton, Geoffrey and others},
	year         = 2009,
	publisher    = {Toronto, ON, Canada}
}

@article{loshchilov2017decoupled,
	title        = {Decoupled weight decay regularization},
	author       = {Loshchilov, Ilya and Hutter, Frank},
	year         = 2017,
	journal      = {arXiv preprint arXiv:1711.05101}
}

@inproceedings{he2016deep,
	title        = {Deep residual learning for image recognition},
	author       = {He, Kaiming and Zhang, Xiangyu and Ren, Shaoqing and Sun, Jian},
	year         = 2016,
	booktitle    = {Proceedings of the IEEE conference on computer vision and pattern recognition},
	pages        = {770--778}
}

@inproceedings{tiny-imagenet,
    title={Squeeze, Recover and Relabel: Dataset Condensation at ImageNet Scale From A New Perspective},
    author={Yin, Zeyuan and Xing, Eric and Shen, Zhiqiang},
    booktitle={Proceedings of the Advances in Neural Information Processing Systems (NeurIPS)},
    year={2023}
}

@software{Howard_Imagenette_2019,
    title={Imagenette: A smaller subset of 10 easily classified classes from Imagenet},
    author={Jeremy Howard},
    year={2019},
    month={March},
    publisher = {GitHub},
    url = {https://github.com/fastai/imagenette}
}

@inproceedings{liu2021swin,
	title        = {Swin transformer: Hierarchical vision transformer using shifted windows},
	author       = {Liu, Ze and Lin, Yutong and Cao, Yue and Hu, Han and Wei, Yixuan and Zhang, Zheng and Lin, Stephen and Guo, Baining},
	year         = 2021,
	booktitle    = {Proceedings of the IEEE/CVF international conference on computer vision},
	pages        = {10012--10022}
}

@inproceedings{devlin2019bert,
	title        = {Bert: Pre-training of deep bidirectional transformers for language understanding},
	author       = {Devlin, Jacob and Chang, Ming-Wei and Lee, Kenton and Toutanova, Kristina},
	year         = 2019,
	booktitle    = {Proceedings of the 2019 conference of the North American chapter of the association for computational linguistics: human language technologies, volume 1 (long and short papers)},
	pages        = {4171--4186}
}

@inproceedings{martens2015optimizing,
	title        = {Optimizing neural networks with kronecker-factored approximate curvature},
	author       = {Martens, James and Grosse, Roger},
	year         = 2015,
	booktitle    = {International conference on machine learning},
	pages        = {2408--2417},
	organization = {PMLR}
}

@inproceedings{grosse2016kronecker,
	title        = {A kronecker-factored approximate fisher matrix for convolution layers},
	author       = {Grosse, Roger and Martens, James},
	year         = 2016,
	booktitle    = {International Conference on Machine Learning},
	pages        = {573--582},
	organization = {PMLR}
}

@article{eschenhagen2023kronecker,
	title        = {Kronecker-factored approximate curvature for modern neural network architectures},
	author       = {Eschenhagen, Runa and Immer, Alexander and Turner, Richard and Schneider, Frank and Hennig, Philipp},
	year         = 2023,
	journal      = {Advances in Neural Information Processing Systems},
	volume       = 36,
	pages        = {33624--33655}
}

@inproceedings{gupta2018shampoo,
	title        = {Shampoo: Preconditioned stochastic tensor optimization},
	author       = {Gupta, Vineet and Koren, Tomer and Singer, Yoram},
	year         = 2018,
	booktitle    = {International Conference on Machine Learning},
	pages        = {1842--1850},
	organization = {PMLR}
}

@article{vyas2024soap,
	title        = {Soap: Improving and stabilizing shampoo using adam},
	author       = {Vyas, Nikhil and Morwani, Depen and Zhao, Rosie and Kwun, Mujin and Shapira, Itai and Brandfonbrener, David and Janson, Lucas and Kakade, Sham},
	year         = 2024,
	journal      = {arXiv preprint arXiv:2409.11321}
}

@article{morwani2024new,
	title        = {A New Perspective on Shampoo's Preconditioner},
	author       = {Morwani, Depen and Shapira, Itai and Vyas, Nikhil and Malach, Eran and Kakade, Sham and Janson, Lucas},
	year         = 2024,
	journal      = {arXiv preprint arXiv:2406.17748}
}

@article{bernstein2024old,
	title        = {Old optimizer, new norm: An anthology},
	author       = {Bernstein, Jeremy and Newhouse, Laker},
	year         = 2024,
	journal      = {arXiv preprint arXiv:2409.20325}
}

@inproceedings{li2026lestd,
	title        = {LeSTD: LLM Compression via Learning-based Sparse Tensor Decomposition},
	author       = {Li, Yi and Guo, Zhichun and Yin, Miao and Li, Bingzhe},
	year         = 2026,
	booktitle    = {The Fourteenth International Conference on Learning Representations}
}

@inproceedings{gu2025tensorllm,
	title        = {TensorLLM: Tensorising Multi-Head Attention for Enhanced Reasoning and Compression in LLMs},
	author       = {Gu, Yuxuan and Zhou, Wuyang and Iacovides, Giorgos and Mandic, Danilo},
	year         = 2025,
	booktitle    = {2025 International Joint Conference on Neural Networks (IJCNN)},
	pages        = {1--8},
	organization = {IEEE}
}

@inproceedings{yuebintd,
	title        = {{TD}-MoE: Tensor Decomposition for MoE Models},
	author       = {Yuebin XU and YANHONG WANG and Xuemei Peng and Hui Zang and Chen Minghao and Pengfei Xia and Zeyi Wen},
	year         = 2026,
	booktitle    = {The Fourteenth International Conference on Learning Representations},
	url          = {https://openreview.net/forum?id=D9cnZNZfxX}
}

@inproceedings{peshekhonov2024training,
	title        = {Training a Tucker model with shared factors: a Riemannian optimization approach},
	author       = {Peshekhonov, Ivan and Arzhantsev, Aleksey and Rakhuba, Maxim},
	year         = 2024,
	booktitle    = {International Conference on Artificial Intelligence and Statistics},
	pages        = {3304--3312},
	organization = {PMLR}
}

@inproceedings{bogachev2026lora,
	title        = {Lo{RA} meets Riemannion: Muon Optimizer for Parametrization-independent Low-Rank Adapters},
	author       = {Vladimir Bogachev and Vladimir Aletov and Alexander Molozhavenko and Denis Bobkov and Vera Soboleva and Aibek Alanov and Maxim Rakhuba},
	year         = 2026,
	booktitle    = {The Fourteenth International Conference on Learning Representations},
	url          = {https://openreview.net/forum?id=WtbXgc9GVA}
}

@inproceedings{mo2025parameter,
	title        = {Parameter and memory efficient pretraining via low-rank riemannian optimization},
	author       = {Mo, Zhanfeng and Huang, Long-Kai and Pan, Sinno Jialin},
	year         = 2025,
	booktitle    = {The Thirteenth International Conference on Learning Representations}
}

@article{pethick2025training,
	title        = {Training Deep Learning Models with Norm-Constrained LMOs},
	author       = {T. Pethick and Wanyun Xie and Kimon Antonakopoulos and Zhenyu Zhu and Antonio José Silveti-Falls and V. Cevher},
	year         = 2025,
	journal      = {International Conference on Machine Learning},
	doi          = {10.48550/arXiv.2502.07529},
	bibsource    = {Semantic Scholar https://www.semanticscholar.org/paper/dd3fabfc7b9c1e866661e777325674a4e96b2466}
}

@inproceedings{xiao2021early,
author = {Xiao, Tete and Singh, Mannat and Mintun, Eric and Darrell, Trevor and Doll\'{a}r, Piotr and Girshick, Ross},
title = {Early convolutions help transformers see better},
year = {2021},
isbn = {9781713845393},
publisher = {Curran Associates Inc.},
address = {Red Hook, NY, USA},
booktitle = {Proceedings of the 35th International Conference on Neural Information Processing Systems},
articleno = {2325},
numpages = {9},
series = {NIPS '21}
}

@inproceedings{kostenetskiy2021hpc,
  title={HPC resources of the higher school of economics},
  author={Kostenetskiy, PS and Chulkevich, RA and Kozyrev, VI},
  booktitle={Journal of Physics: Conference Series},
  volume={1740},
  number={1},
  pages={012050},
  year={2021},
  organization={IOP Publishing}
}

@article{molozhavenkosfett,
	title        = {Optimization on the extended tensor-train manifold with shared factors},
	author       = {Alexander Molozhavenko and Maxim V. Rakhuba},
	year         = 2025,
	journal      = {Computational and Applied Mathematics},
	volume       = 45,
	url          = {https://api.semanticscholar.org/CorpusID:280950046}
}

@article{gotze2021concentration,
author = {Friedrich Gotze and Holger Sambale and Arthur Sinulis},
title = {{Concentration inequalities for polynomials in α-sub-exponential random variables}},
volume = {26},
journal = {Electronic Journal of Probability},
number = {none},
publisher = {Institute of Mathematical Statistics and Bernoulli Society},
pages = {1 -- 22},
year = {2021},
doi = {10.1214/21-EJP606},
URL = {https://doi.org/10.1214/21-EJP606}
}
\newpage
\appendix
\onecolumn

\renewcommand{\theproposition}{\Alph{section}.\arabic{proposition}}
\section{Primal norm}\label{sec:appendix_primal_norm}

\begin{theorem}[Primal relaxed spectral norm]\label{theorem:primal_relaxed_spectral_norm}
    Let $X \in \mathbb{R}^{n_1 \times \ldots \times n_d}$ then 
    \begin{equation}\label{eq:primal-prob}
        \|X\|_\Sigma = \min_{\substack{(\Xi^\tau)_{\tau \in \Tau} \\
                \sum\limits_{\tau \in \Tau} \Xi^\tau = X}}
        \;\;
        \sum_{\tau \in \Tau} \left\| \Xi^\tau_{(\tau)} \right\|_2,
    \end{equation}
    where $\Xi^\tau \in \mathbb{R}^{n_1 \times \dotsc \times n_d}, \tau \in \Tau$ are auxiliary variables.
\end{theorem}
\begin{proof}
Since $\Tau$ is finite, the constraint
\begin{equation}
    \|M\|_{\Sigma}^{\dagger}=\max_{\tau\in\Tau}\|M_{(\tau)}\|_*\le 1
\end{equation}
is equivalent to the collection of constraints
\begin{equation}
    \|M_{(\tau)}\|_* \le 1
    \qquad
    \forall \tau\in\Tau.
\end{equation}
Hence, by introducing auxiliary tensors $A^\tau$ that duplicate $M$, we can rewrite the definition of the dual norm as
\begin{equation}
    \|X\|_{\Sigma}
=
\max_{M,(A^\tau)_{\tau\in\Tau}}
\left\{
\langle M,X\rangle
:\;
M=A^\tau,\ \|A^\tau_{(\tau)}\|_*\le 1\ \forall \tau\in\Tau
\right\}.
\end{equation}
Equivalently, we consider the minimization problem
\begin{equation}
    \min_{M,(A^\tau)_{\tau\in\Tau}}
    \left\{
    -\langle M,X\rangle
    \colon
    M=A^\tau,\ \|A^\tau_{(\tau)}\|_*\le 1\ \forall \tau\in\Tau
    \right\}.
\end{equation}
Its optimal value is $-\|X\|_{\Sigma}$.

This is a finite-dimensional convex problem, and Slater's condition holds (for instance, $M=0$ and $A^\tau=0$ strictly satisfy all inequality constraints). Therefore, strong duality applies.

Introduce dual variables $\Xi^\tau$ for the constraints $M=A^\tau$ and nonnegative multipliers $\lambda^\tau$ for the constraints $\|A^\tau_{(\tau)}\|_*\le 1$. The Lagrangian is
\begin{equation}
     \mathcal L(M,A,\lambda,\Xi)
    =
    -\langle M,X\rangle
    +
    \sum_{\tau\in\Tau}
    \Bigl[
    \lambda^\tau\bigl(\|A^\tau_{(\tau)}\|_* -1\bigr)
    +
    \langle \Xi^\tau, M-A^\tau\rangle
    \Bigr].   
\end{equation}
Rearranging terms gives
\begin{equation}
    \mathcal L(M,A,\lambda,\Xi)
    =
    \left\langle -X+\sum_{\tau\in\Tau}\Xi^\tau,\; M\right\rangle
    +
    \sum_{\tau\in\Tau}
    \Bigl[
    \lambda^\tau \|A^\tau_{(\tau)}\|_*
    -
    \langle \Xi^\tau, A^\tau\rangle
    -
    \lambda^\tau
    \Bigr].
\end{equation}
We now minimize over the primal variables. First, if
\begin{equation}
    \sum_{\tau\in\Tau}\Xi^\tau \neq X,
\end{equation}
then the term linear in $M$ is nonzero, and minimizing over $M$ yields $-\infty$. Thus the dual function is finite only if
\begin{equation}
    \sum_{\tau\in\Tau}\Xi^\tau = X.
\end{equation}

Assume from now on that this constraint holds. Then the minimization over $A^\tau$ separates across $\tau$. For each $\tau$, we use that unfolding preserves the Frobenius inner product:
\begin{equation}
    \langle \Xi^\tau, A^\tau\rangle
    =
    \langle \Xi^\tau_{(\tau)}, A^\tau_{(\tau)}\rangle.
\end{equation}
Next, recall the standard duality between the nuclear and spectral norms:
\begin{equation}
    \langle U,V\rangle \le \|U\|_2\,\|V\|_*. 
\end{equation}
Applying this with $U=\Xi^\tau_{(\tau)}$ and $V=A^\tau_{(\tau)}$ gives
\begin{equation}
    \langle \Xi^\tau, A^\tau\rangle
    \le
    \|\Xi^\tau_{(\tau)}\|_2\,\|A^\tau_{(\tau)}\|_*.
\end{equation}

Therefore,
\begin{equation}
    \lambda^\tau \|A^\tau_{(\tau)}\|_*
    -
    \langle \Xi^\tau, A^\tau\rangle
    -
    \lambda^\tau
    \ge
    \bigl(\lambda^\tau-\|\Xi^\tau_{(\tau)}\|_2\bigr)\|A^\tau_{(\tau)}\|_*
    -\lambda^\tau.
\end{equation}

If $\lambda^\tau \ge \|\Xi^\tau_{(\tau)}\|_2$,
the right-hand side is bounded below by $-\lambda^\tau$, and this bound is attained at $A^\tau=0$. Hence the minimum over $A^\tau$ equals $-\lambda^\tau$.

If instead $\lambda^\tau < \|\Xi^\tau_{(\tau)}\|_2$,
then the coefficient in front of $\|A^\tau_{(\tau)}\|_*\!$ is negative. Using the fact that the spectral norm is dual to the nuclear norm, one can choose a matrix $B^\tau$ with
\begin{equation}
    \|B^\tau\|_*=1,
    \qquad
    \langle \Xi^\tau_{(\tau)},B^\tau\rangle=\|\Xi^\tau_{(\tau)}\|_2.
\end{equation}

Let $\widetilde A^\tau$ be the tensor whose $\tau$-matricization is $B^\tau$. Evaluating the Lagrangian at $t\widetilde A^\tau$ and sending $t\to\infty$ shows that the infimum is then $-\infty$.

Thus the dual function is
\begin{equation}
   \mathcal G(\lambda,\Xi)
    =
    \begin{cases}
    -\sum\limits_{\tau\in\Tau}\lambda^\tau,
    &
    \sum\limits_{\tau\in\Tau}\Xi^\tau=X
    \ \text{and}\
    \lambda^\tau\ge \|\Xi^\tau_{(\tau)}\|_2
    \ \forall\tau\in\Tau,
    \\[0.6em]
    -\infty,
    & \text{otherwise}.
    \end{cases} 
\end{equation}

The dual problem is therefore
\begin{equation}
    \max_{\lambda,\Xi}\mathcal G(\lambda,\Xi)
    =
    -
    \min_{\substack{\sum_{\tau\in\Tau}\Xi^\tau=X\\
    \lambda^\tau\ge \|\Xi^\tau_{(\tau)}\|_2,\ \forall\tau\in\Tau}}
    \sum_{\tau\in\Tau}\lambda^\tau.
\end{equation}
For fixed $(\Xi^\tau)_{\tau\in\Tau}$, the minimum over $\lambda^\tau$ is attained at
\begin{equation}
\lambda^\tau=\|\Xi^\tau_{(\tau)}\|_2.
\end{equation}

Hence the dual optimal value is
\begin{equation}
    -
    \min_{\substack{(\Xi^\tau)_{\tau\in\Tau}\\
    \sum_{\tau\in\Tau}\Xi^\tau=X}}
    \sum_{\tau\in\Tau}\|\Xi^\tau_{(\tau)}\|_2.
\end{equation}

Finally, by strong duality,
\begin{equation}
    -\|X\|_{\Sigma}
    =
    -
    \min_{\substack{(\Xi^\tau)_{\tau\in\Tau}\\
    \sum_{\tau\in\Tau}\Xi^\tau=X}}
    \sum_{\tau\in\Tau}\|\Xi^\tau_{(\tau)}\|_2,
\end{equation}
and multiplying by $-1$ yields
\begin{equation}
    \|X\|_{\Sigma}
=
\min_{\substack{(\Xi^\tau)_{\tau\in\Tau}\\
\sum_{\tau\in\Tau}\Xi^\tau=X}}
\sum_{\tau\in\Tau}\|\Xi^\tau_{(\tau)}\|_2.
\end{equation}
This proves the claim.
\end{proof}
\begin{lemma}\label{lemma:dual_ineqs}
    Let $\|\cdot\|_A, \|\cdot\|_B$ be any norms defined on the same finite-dimensional linear space $V$ and $\|\cdot\|_A^\dagger, \|\cdot\|_B^\dagger$ are their duals. Let also
    \begin{equation}
        \|v\|_A \leq \|v\|_B, \quad v \in V,
        \label{eq:lemma:aux1}
    \end{equation}
    then
    \begin{equation}
        \|u\|^\dagger_A \geq \|u\|^\dagger_B, \quad u \in V^*.
    \end{equation}
\end{lemma}
\begin{proof}
    By definition of the dual norm:
    \begin{equation}
        \|u\|^\dagger_A = \max\limits_{v\in B_A(0,1)} \langle v, u\rangle, \quad 
         \|u\|^\dagger_B = \max\limits_{v\in B_B(0,1)} \langle v, u\rangle, 
        \label{eq:lemma:aux2}
    \end{equation}
    where
    \begin{equation}
        B_A(0, 1) \colonequals \left\{v \in V \, | \, \|v\|_A \leq 1\right\}, \quad 
        B_B(0, 1) \colonequals \left\{v \in V \, | \, \|v\|_B \leq 1\right\}.
    \end{equation}
    From Eq.~\eqref{eq:lemma:aux1} it follows that $B_B(0,1) \subseteq B_A(0,1)$, therefore maximum over $B_A(0,1)$ in Eq.~\eqref{eq:lemma:aux2} is not smaller than maximum over $B_B(0,1)$ yielding the proposed inequality. 
\end{proof}

\begin{proposition}\label{proposition:tensorion_and_spectral_norm_ineqs_proof}
  Let $X\in \mathbb{R}^{n_1 \times \dots \times n_d}$ be any tensor then $\|X\|_\Sigma^\dagger \leq \|X\|_\sigma^\dagger$ and
  \begin{equation}
    \|X\|_\sigma \leq \|X\|_\Sigma \le \min_{\tau \in \Tau} \|X_{(\tau)}\|_2.
    \label{eq:tensorion_dual_bounds_nuclear_proof}
  \end{equation}
\end{proposition}
\begin{proof}
    We begin by comparing the dual norms. By the definition of $\|\cdot\|_\Sigma^\dagger$ (see Eq.~\eqref{eq:tensorion_norm}),
    \begin{equation}
        \|M\|_\Sigma^\dagger = \max_{\tau \in \Tau} \|M_{(\tau)}\|_* .
    \end{equation}
    Applying Lemma~\ref{lemma:dual_ineqs} to $\|X_{(\tau)}\|_2 \ge \|X\|_\sigma$ for each unfolding yields
    \begin{equation}
        \|M_{(\tau)}\|_* \le \|M\|_*
        \quad \forall \tau \in \Tau.
    \end{equation}
    Therefore,
    \begin{equation}
        \|M\|_\Sigma^\dagger
        =
        \max_{\tau \in \Tau} \|M_{(\tau)}\|_*
        \le
        \|M\|_*.
    \end{equation}
    Finally, the tensor nuclear norm is dual to the tensor spectral norm \citep[Lemma $21$]{limblind}, so $\|M\|_* = \|M\|_\sigma^\dagger$.
    Hence $\|M\|_\Sigma^\dagger \le \|M\|_\sigma^\dagger$.
    
    We now turn to the lower bound on $\|X\|_\Sigma$. Applying Lemma~\ref{lemma:dual_ineqs} to 
    \begin{equation}
        \|M\|_\Sigma^\dagger \le \|M\|_\sigma^\dagger
    \end{equation}
    yields
    \begin{equation}
        \|X\|_\sigma \le \|X\|_\Sigma
    \end{equation}
    
    It remains to establish the upper bound. By the variational representation of $\|\cdot\|_\Sigma$ (see Eq.~\eqref{eq:primal-prob}),
    \begin{equation}
        \|X\|_\Sigma
        =
        \min_{\substack{(\Xi^\tau)_{\tau \in \Tau}\\ \sum_{\tau \in \Tau} \Xi^\tau = X}}
        \sum_{\tau \in \Tau} \|\Xi^\tau_{(\tau)}\|_2.
    \end{equation}
    Fix any $\tau' \in \Tau$, and consider the feasible decomposition
    \begin{equation}
        \Xi^\tau =
        \begin{cases}
        X, & \tau = \tau',\\
        0, & \tau \ne \tau'.
        \end{cases}
    \end{equation}
    Substituting this choice into the objective yields
    \begin{equation}
        \|X\|_\Sigma
        \le
        \sum_{\tau \in \Tau} \|\Xi^\tau_{(\tau)}\|_2
        =
        \|\Xi^{\tau'}_{(\tau')}\|_2
        =
        \|X_{(\tau')}\|_2.
    \end{equation}
    Since $\tau'$ was arbitrary, we conclude that
    \begin{equation}
        \|X\|_\Sigma \le \min_{\tau \in \Tau} \|X_{(\tau)}\|_2.
    \end{equation}
    This completes the proof.
\end{proof}

\begin{proposition}\label{proposition:tensorion_norm_LMO_solution_proof}
    The solution set of the LMO in Eq.~\eqref{eq:tensorion_norm_LMO} is the subdifferential of the relaxed spectral dual norm at point $M$:
  \begin{equation}
    X^{\text{opt}}\in \partial \|M\|_\Sigma^\dagger = \conv_{\tau \in I}\,\left( \mathrm{fold}_\tau \left( \partial \|M_{(\tau)} \|_* \right) \right),
    \label{eq:subdiff_dual_norm_sol_proof}
  \end{equation}
  where  $\conv{(\cdot)}$ denotes a convex hull and
  \begin{equation}
    I = \left\{ \tau \,\middle|\, 
    \|M_{(\tau)}\|_* = \|M\|_\Sigma^\dagger \right\} \subseteq 2^{\{1, \dotsc, d\}}
    \label{eq:maximizing_indices_proof}
  \end{equation}
  is a nonempty set of unfolding indices attaining the maximal nuclear norm.
\end{proposition}
\begin{proof}
    First, by definition of the relaxed spectral dual norm (see Eq.~\eqref{eq:tensorion_norm}),
    \begin{equation}
    \|M\|_\Sigma^\dagger
    =
    \max_{\tau \in \Tau} \|M_{(\tau)}\|_*.
    \end{equation}
    For every \(\tau \in \Tau\), define
    \begin{equation}
    f_\tau(M) := \|M_{(\tau)}\|_*.
    \end{equation}
    Since the unfolding map \(M \mapsto M_{(\tau)}\) is linear and the nuclear norm is convex, each function \(f_\tau\) is convex in \(M\). Because the index set \(\Tau\) is finite, we may apply the Dubovitskii-Milyutin subdifferential formula (see, \eg \cite{girsanov2012lectures}) for the maximum of finitely many convex functions:
    \begin{equation}
    \partial \left( \max_{\tau \in \Tau} f_\tau(M) \right)
    =
    \operatorname{conv}\left(
        \bigcup_{\tau \in I} \partial f_\tau(M)
    \right),
    \end{equation}
    where
    \begin{equation}
    I
    =
    \left\{
        \tau \in \Tau \;\middle|\; f_\tau(M) = \max_{\theta \in \Tau} f_\theta(M)
    \right\}.
    \end{equation}
    Substituting the definition of \(f_\tau\), we obtain
    \begin{equation}
    \partial \|M\|_\Sigma^\dagger
    =
    \conv_{\tau \in I}\,\left( \mathrm{fold}_\tau \left( \partial \|M_{(\tau)} \|_* \right) \right),
    \end{equation}
    with
    \begin{equation}
    I
    =
    \left\{
        \tau \in \Tau \;\middle|\; \|M_{(\tau)}\|_* = \|M\|_\Sigma^\dagger
    \right\}.
    \end{equation}
    
    It remains to identify this subdifferential with the solution set of the LMO in Eq.~\eqref{eq:tensorion_norm_LMO}. For any norm $\|\cdot\|_A$ and its dual norm $\|\cdot\|_A^\dagger$ \citep[Theorem~$23.5$]{rockafellar1997convex}
    \begin{equation}
    \partial \|x\|_A
    =
    \operatorname{Argmax}\left\{
        \langle z, x \rangle \;\middle|\; \|z\|_A^\dagger\le 1
    \right\}, \quad x \neq 0.
    \end{equation}
    Applying this general result with
    \begin{equation}
    \|\cdot\|_A = \|\cdot\|_\Sigma^\dagger,
    \qquad
    \|\cdot\|_A^\dagger = \|\cdot\|_\Sigma,
    \end{equation}
    we obtain
    \begin{equation}
    \partial \|M\|_\Sigma^\dagger
    =
    \operatorname{Argmax}\left\{
        \langle X, M \rangle \;\middle|\; \|X\|_\Sigma \le 1
    \right\}.
    \end{equation}
    Combining this identity with the representation established above completes the proof.
\end{proof}

\section{High-Probability Nuclear Norm Bound}\label{sec:proof_randomized_bound}

\begin{proposition}[High-Probability Nuclear Norm Bound]\label{prop:randomized_bound_app}
    Let the entries of $M\in\mathbb R^{n_1\times\cdots\times n_d}$ be i.i.d.\ copies of a
mean-zero random variable $Z$ with $\|Z\|_{\psi_1}\le K$, and let $N=\prod_{i=1}^{d} n_i$.
Consider any unfolding $M_{(\tau)} \in \mathbb{R}^{m_\tau \times n_\tau}$ with $m_\tau \leq n_\tau$
such that $m_{\tau}n_{\tau} = \prod_{i=1}^d n_i$.
Then there exists an absolute constant $C_0 > 0$ such that, with $C:=C_0K$, for any $\delta\in(0,1)$, with probability at least $1-\delta$,
simultaneously for every unfolding $\tau$,
\begin{equation}
\label{eq:prob_bound}
\|M_{(\tau)}\|_*\le C \Bigl(m_\tau \sqrt{n_\tau} + \sqrt{m_\tau} \Bigl(\sqrt{\log\tfrac2\delta} + \log \tfrac{2}{\delta}\Bigr)\Bigr).
\end{equation}
Moreover, for a fixed total dimension $m_{\tau}n_{\tau}$, the value of this upper bound is
maximized by the most balanced admissible unfolding, i.e.\ by the $\tau$ with the
largest $m_\tau$ subject to $m_\tau\le n_\tau$.
\end{proposition}

\begin{proof}
    Let $N = \prod_{i=1}^d n_i$ and write $L := \log(2/\delta)$. Since unfolding only permutes
    the entries of the tensor,
    \begin{equation}
    \|M_{(\tau)}\|_F = \|M\|_F = \|\operatorname{vec}(M)\|_2
    \end{equation}
    for every admissible $\tau \in \mathcal T$.

    By the standard relation between the nuclear and Frobenius norms for $A \in \mathbb{R}^{m \times n}$
    with singular values $\sigma_1 \ge \cdots \ge \sigma_r > 0$, where $r = \mathrm{rank}(A) \le \min(m,n)$,
    $\|A\|_* = \sum_{i=1}^r \sigma_i$ and $\|A\|_F = ( \sum_{i=1}^r \sigma_i^2 )^{1/2}$,
    the Cauchy--Schwarz inequality gives
    \begin{equation}
    \|A\|_*
    =
    \sum_{i=1}^r \sigma_i \cdot 1
    \le
    \left( \sum_{i=1}^r \sigma_i^2 \right)^{1/2}
    \left( \sum_{i=1}^r 1^2 \right)^{1/2}
    =
    \sqrt{r}\,\|A\|_F
    \le
    \sqrt{\min(m,n)}\,\|A\|_F.
    \end{equation}
    In particular, since $m_\tau \le n_\tau$,
    \begin{equation}\label{eq:cs_step}
    \|M_{(\tau)}\|_* \le \sqrt{m_\tau}\,\|M_{(\tau)}\|_F = \sqrt{m_\tau}\,\|\operatorname{vec}(M)\|_2 .
    \end{equation}

    Therefore it suffices to control $\|\operatorname{vec}(M)\|_2$, and we can reduce the problem to
    controlling the norm of a random vector. Denote by $Z_1,\dots,Z_N$ the entries of $M$,
    and let $\sigma^2:=\mathbb E Z^2$ and $S:=\|\operatorname{vec}(M)\|_2^2=\sum_{j=1}^N Z_j^2$.
    As $x^2/2\le e^{x}$ for $x\ge 0$ and $\mathbb E\exp(|Z|/K)\le 2$ we get that $\sigma\le 2K$.
    Using the Hanson--Wright inequality for $\alpha$-sub-exponential random vectors,
    \citep[Cor.~1.4]{gotze2021concentration}, with $\alpha=1$, $x=\operatorname{vec}(M)$ and
    $A=I_N$, we obtain $\mathbb E\,x^\top I_N x = N\sigma^2$, and for every $u\ge0$:
    \begin{equation}\label{eq:hw}
    \mathbb P\bigl(S-N\sigma^{2}\ge u\bigr)
    \le
    2\exp\!\left(-c\min\!\left(\frac{u^{2}}{K^{4}N},\ \frac{\sqrt{u}}{K}\right)\right),
    \end{equation}
    for some constant $c > 0$.
    Now fix $\delta\in(0,1)$ and set
    \begin{equation}
    u_\delta := C_1K^{2}\bigl(\sqrt{NL}+L^{2}\bigr),
    \qquad C_1:=\max\bigl(c^{-1/2},c^{-2}\bigr),
    \end{equation}
    so that $u_\delta^{2}/(K^4N)\ge C_1^{2}L\ge c^{-1}L$ and
    $\sqrt{u_\delta}/K\ge\sqrt{C_1}\,L\ge c^{-1}L$, whence the right-hand side of
    \eqref{eq:hw} is at most $\delta$. Then, with probability at least $1-\delta$,
    $S\le N\sigma^2+u_\delta$, and, applying $\sqrt{a+b}\le\sqrt a+\sqrt b$ twice,
    \begin{equation}
    \|\operatorname{vec}(M)\|_2
    \le
    \sigma\sqrt N+\sqrt{C_1}\,K\bigl(N^{1/4}L^{1/4}+L\bigr).
    \end{equation}
    By AM--GM, $N^{1/4}L^{1/4} \le \tfrac{1}{2} (\sqrt N+\sqrt L)$, so that,
    using $\sigma\le 2K$,
    \begin{equation}\label{eq:vec_bound}
    \|\operatorname{vec}(M)\|_2\le C_0K\bigl(\sqrt N+\sqrt L+L\bigr),
    \qquad C_0:=2+\sqrt{C_1}.
    \end{equation}
    Since the event in \eqref{eq:vec_bound} does not depend on $\tau$, the bounds hold on it
    simultaneously for every admissible unfolding.

    Combining \eqref{eq:vec_bound} with \eqref{eq:cs_step} and using $N=m_\tau n_\tau$,
    so that $\sqrt{m_\tau N}=m_\tau\sqrt{n_\tau}$, we obtain
    \begin{equation}
    \|M_{(\tau)}\|_*
    \le
    \sqrt{m_\tau}\,\|\operatorname{vec}(M)\|_2
    \le
    C_0K\Bigl(m_\tau\sqrt{n_\tau}+\sqrt{m_\tau}\bigl(\sqrt{L}+L\bigr)\Bigr),
    \end{equation}
    which is \eqref{eq:prob_bound} with $C=C_0K$.

    Finally, let us justify the statement about the ``most balanced'' unfolding. For fixed total
    dimension $N$ and given the constraint $m_\tau\le n_\tau$, we have $n_\tau=N/m_\tau$, and hence
    \begin{equation}
    m_\tau \sqrt{n_\tau} = m_\tau \sqrt{\frac{N}{m_\tau}} = \sqrt{N m_\tau}.
    \end{equation}
    Thus the leading term is monotone increasing in $m_\tau$, and the deviation term, being
    proportional to $\sqrt{m_\tau}$, is monotone increasing in $m_\tau$ as well; the constraint
    $m_\tau\le n_\tau$ is equivalent to $m_\tau\le\sqrt N$. Since admissible unfoldings are
    determined by partitions of the tensor modes, not every factorization of $N$ need be
    realizable. Therefore, among the admissible unfoldings, the largest upper bound is attained by
    the unfolding with the largest possible value of $m_\tau$ subject to $m_\tau \le n_\tau$,
    namely by the most balanced admissible unfolding.
\end{proof}

\section{Additional experiments}

\subsection{CIFAR-10 and CIFAR-100}
\label{appc:cifar}

This section presents additional training results for ResNet-18 and ResNet-34 in the CIFAR10 and CIFAR100 datasets respectively. Unlike previous experiments, here we also evaluate the Tensorion Algorithm~\ref{alg:one_step_of_tensorion} using the AdamW optimizer for the non-tensor blocks; this setting is denoted as $\mathtt{tensorion\_adamw}$ in Fig.~\ref{fig:cifar10-4} and~\ref{fig:cifar100-4}.

It is also important to analyze how the algorithm’s behavior depends on the optimizer chosen for the non-tensor parameters of the model. Indeed, the reported experiments show that the behavior strongly depends on this choice. When SGD is used, we observe improvements in both convergence and accuracy compared to applying SGD to all model parameters. For AdamW, we observe a roughly similar picture, where both algorithms achieve approximately the same performance.

\begin{figure*}[h]
    \centering

    \begin{subfigure}[t]{\linewidth}
        \centering
        \includegraphics[width=0.8\linewidth]{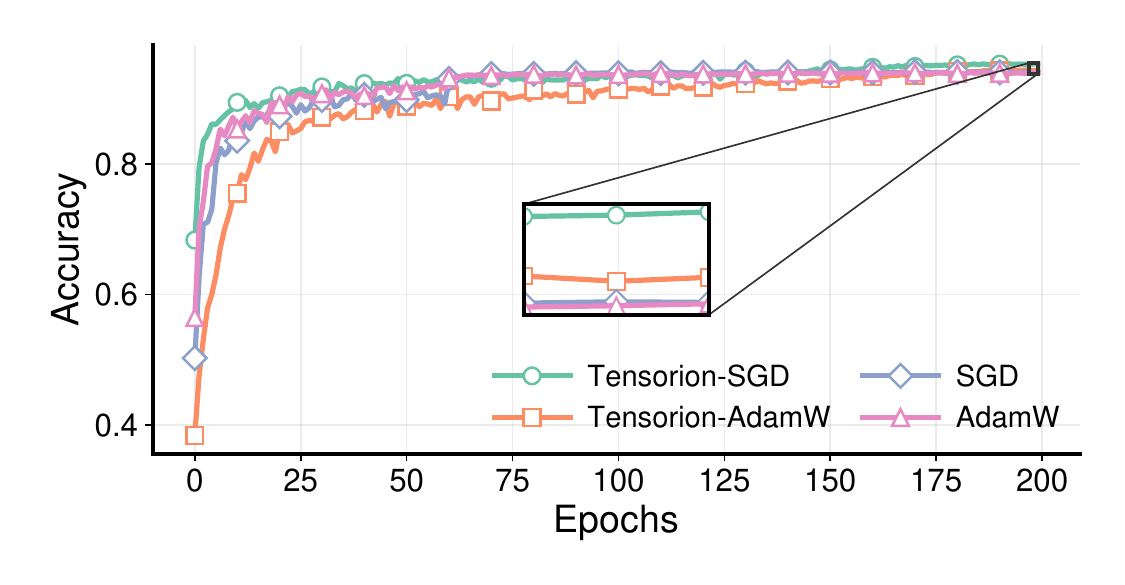}
        \caption{Test accuracy vs Epoch}
        \label{fig:cifar10-4-accuracy}
    \end{subfigure}

    \begin{subfigure}[t]{0.49\linewidth}
        \centering
        \includegraphics[width=\linewidth]{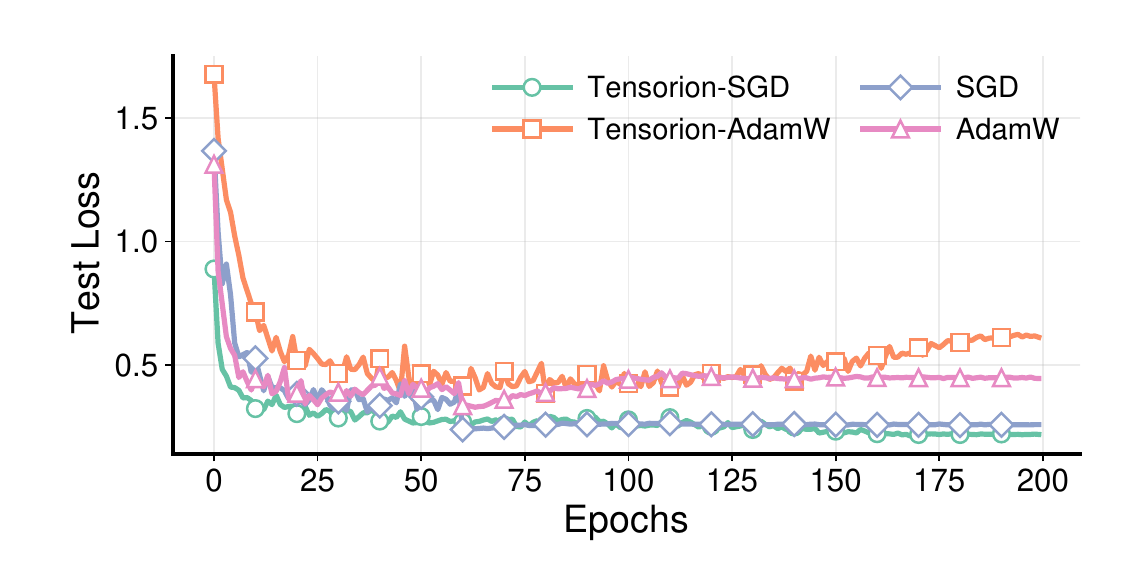}
        \caption{Test loss vs Epoch}
        \label{fig:cifar10-4-test-loss}
    \end{subfigure}\hfill
    \begin{subfigure}[t]{0.49\linewidth}
        \centering
        \includegraphics[width=\linewidth]{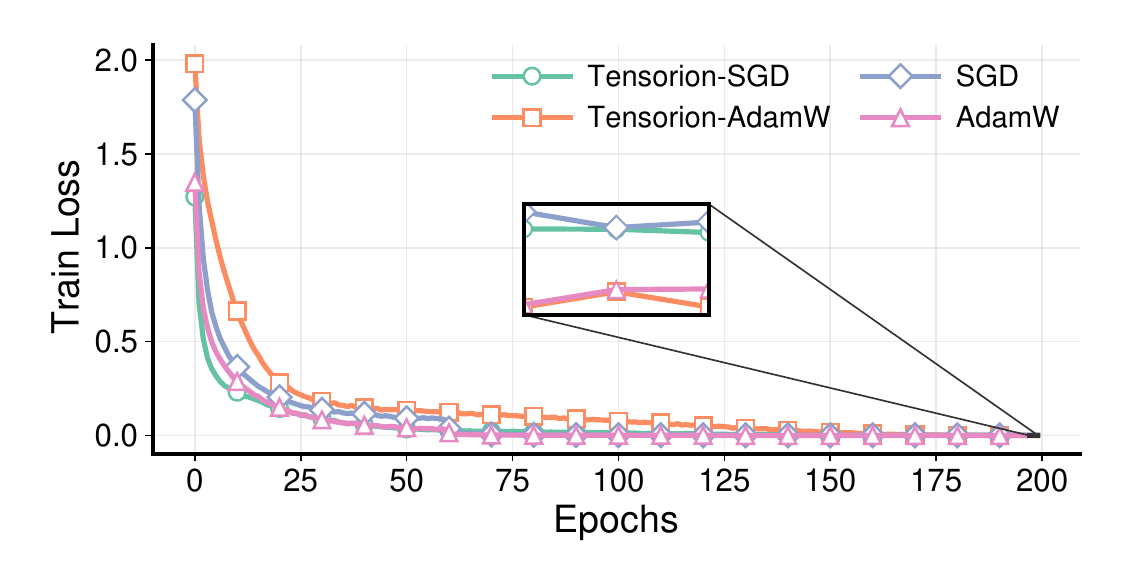}
        \caption{Train loss vs Epoch}
        \label{fig:cifar10-4-train-loss}
    \end{subfigure}

    \caption{CIFAR-10 results for ResNet-18 trained for $200$ epochs under the standard CIFAR pipeline of~\cite{2023mmpretrain}. We report mean \(\pm\) std over multiple seeds for test accuracy (Figure~\ref{fig:cifar10-4-accuracy}), test loss (Figure~\ref{fig:cifar10-4-test-loss}), and train loss (Figure~\ref{fig:cifar10-4-train-loss}). Tensorion is used only for tensor-valued blocks with order \(d\ge 3\) (e.g., convolutional kernels), while all remaining parameters (\(d\le 2\)) are optimized with a first-order method (SGD or AdamW).}
    \label{fig:cifar10-4}
\end{figure*}

\begin{figure*}[h]
    \centering

    \begin{subfigure}[t]{\linewidth}
        \centering
        \includegraphics[width=0.8\linewidth]{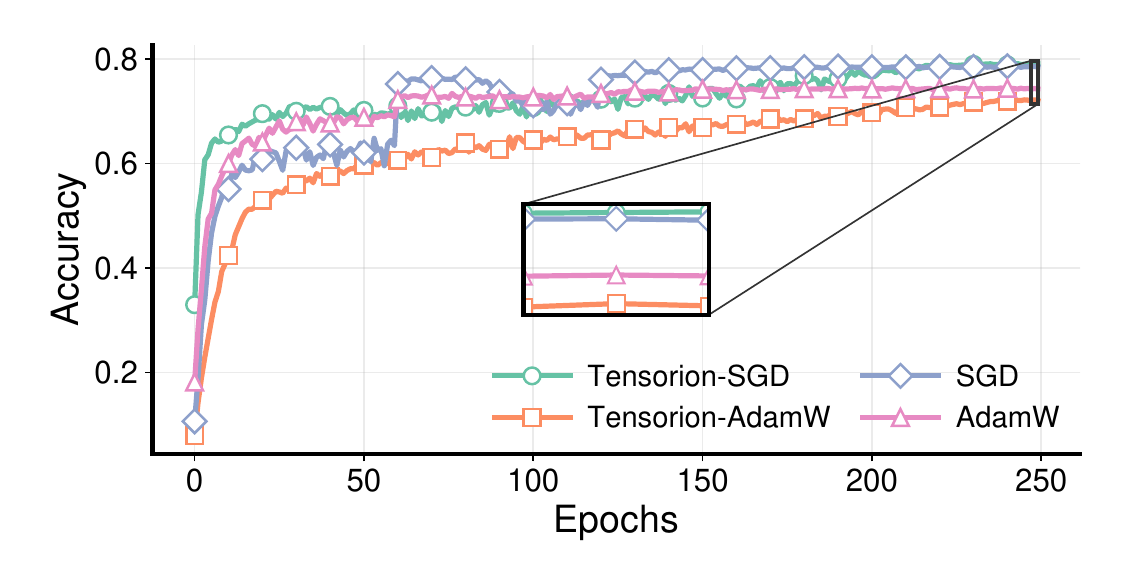}
        \caption{Test accuracy vs Epoch}
        \label{fig:cifar100-4-accuracy}
    \end{subfigure}

    \begin{subfigure}[t]{0.49\linewidth}
        \centering
        \includegraphics[width=\linewidth]{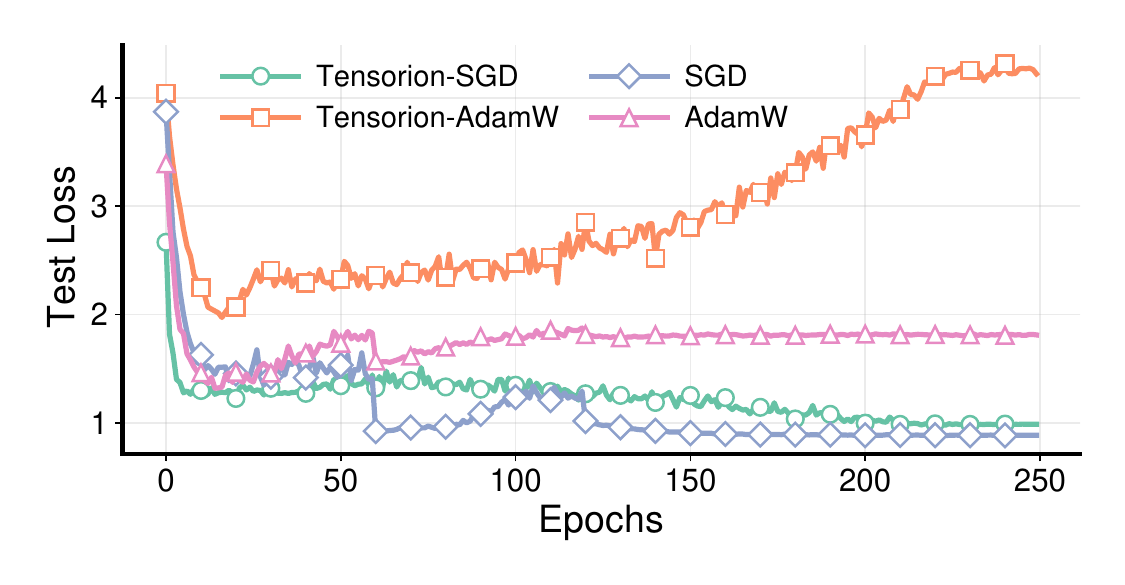}
        \caption{Test loss vs Epoch}
        \label{fig:cifar100-4-test-loss}
    \end{subfigure}\hfill
    \begin{subfigure}[t]{0.49\linewidth}
        \centering
        \includegraphics[width=\linewidth]{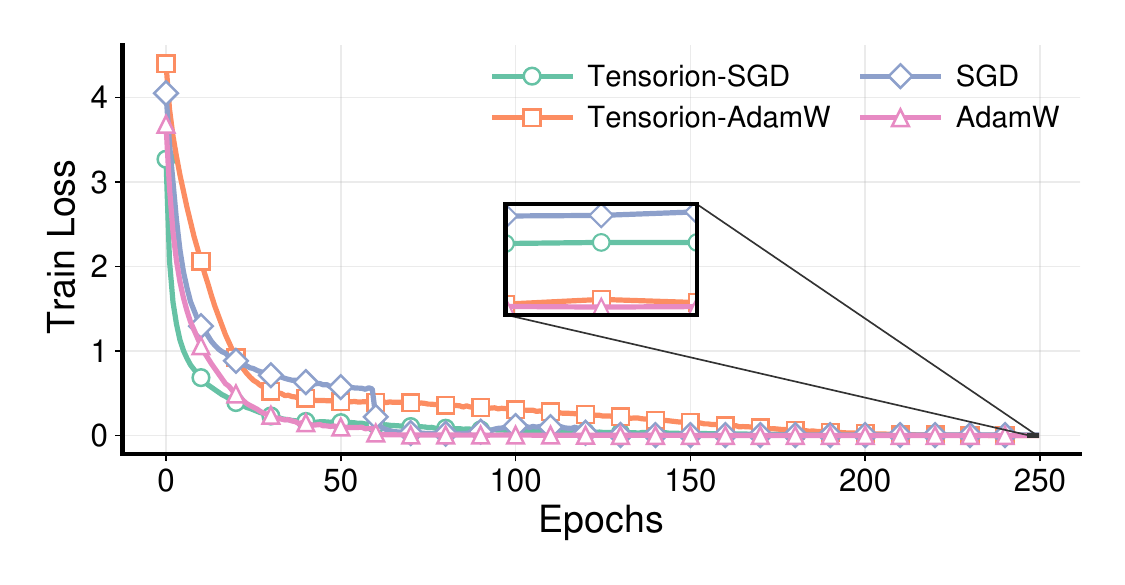}
        \caption{Train loss vs Epoch}
        \label{fig:cifar100-4-train-loss}
    \end{subfigure}

    \caption{CIFAR-100 results for ResNet-34 trained for 250 epochs under the standard CIFAR pipeline of~\cite{2023mmpretrain}. We report mean \(\pm\) std over multiple seeds for test accuracy (Figure~\ref{fig:cifar100-4-accuracy}), test loss (Figure~\ref{fig:cifar100-4-test-loss}), and train loss (Figure~\ref{fig:cifar100-4-train-loss}). Tensorion is used only for tensor-valued blocks with order \(d\ge 3\) (e.g., convolutional kernels), while all remaining parameters (\(d\le 2\)) are optimized with a first-order method (SGD or AdamW).}
    \label{fig:cifar100-4}
\end{figure*}

\subsection{ViT \& Swin models}
\label{sec:details_vit_swin}
Here we present some additional details on evaluations on transformer-based models. We initialize all the models from pretrained checkpoints available in the \texttt{timm} library and finetune the models for 5 epochs on Imagenette dataset. Tensorion is only applied to convolutional layers, while we revert to either AdamW or SGD for other layers to isolate the effects. Notably, applying Tensorion to weight matrices in Transformer blocks would recover Muon updates. We fix weight decay and LR scheduler parameters for this experiment, and only sweep over optimal learning rates for the respective choice of optimizers. We report best runs for each choice of the optimizer in \ref{fig:all_model_scores_swin} and \ref{fig:all_model_scores_vit}.

\begin{figure}[h!tp]
    \centering
    
    \begin{subfigure}{0.49\linewidth}
        \centering
        \includegraphics[width=\linewidth]{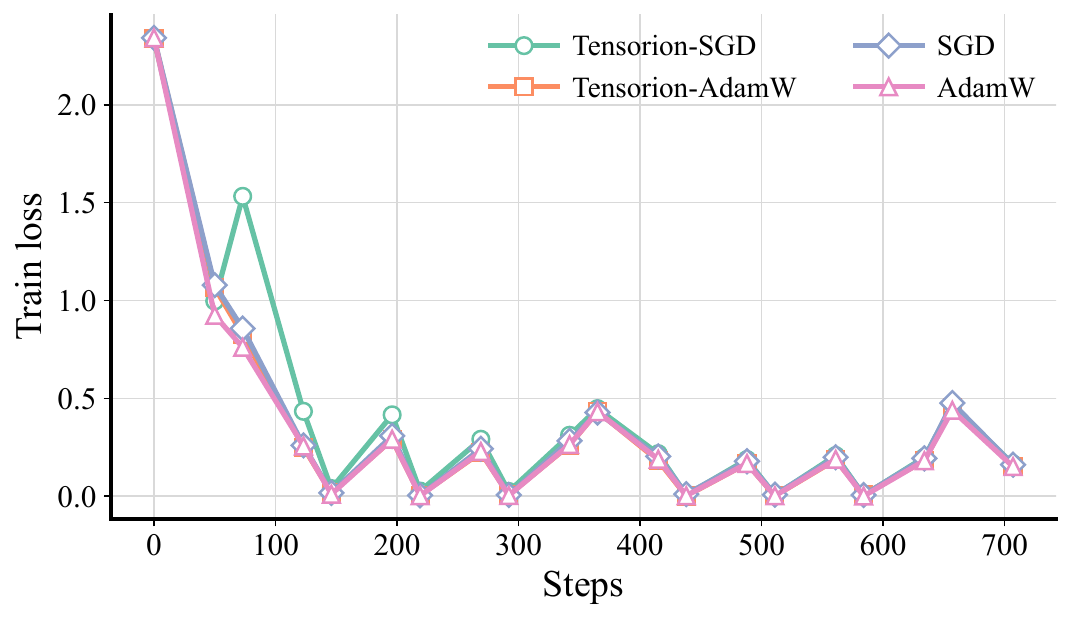}
        \caption{Swin Tiny Train Loss}
        \label{fig:swin_tiny_loss}
    \end{subfigure}
    \begin{subfigure}{0.49\linewidth}
        \centering
        \includegraphics[width=\linewidth]{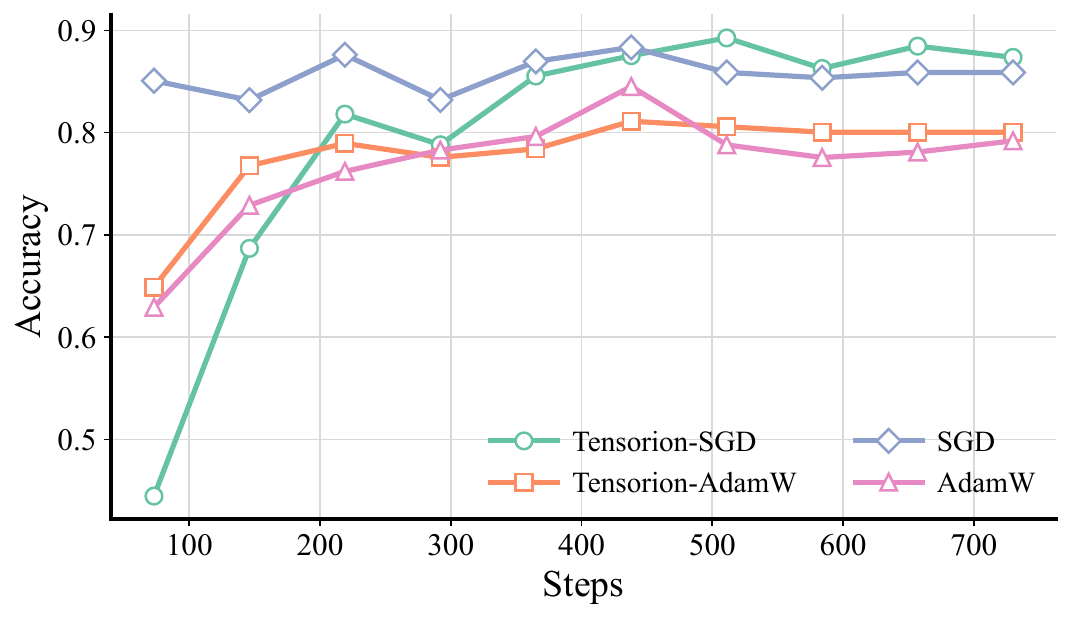}
        \caption{Swin Tiny Accuracy}
        \label{fig:swin_tiny_acc}
    \end{subfigure}

    \begin{subfigure}{0.49\linewidth}
        \centering
        \includegraphics[width=\linewidth]{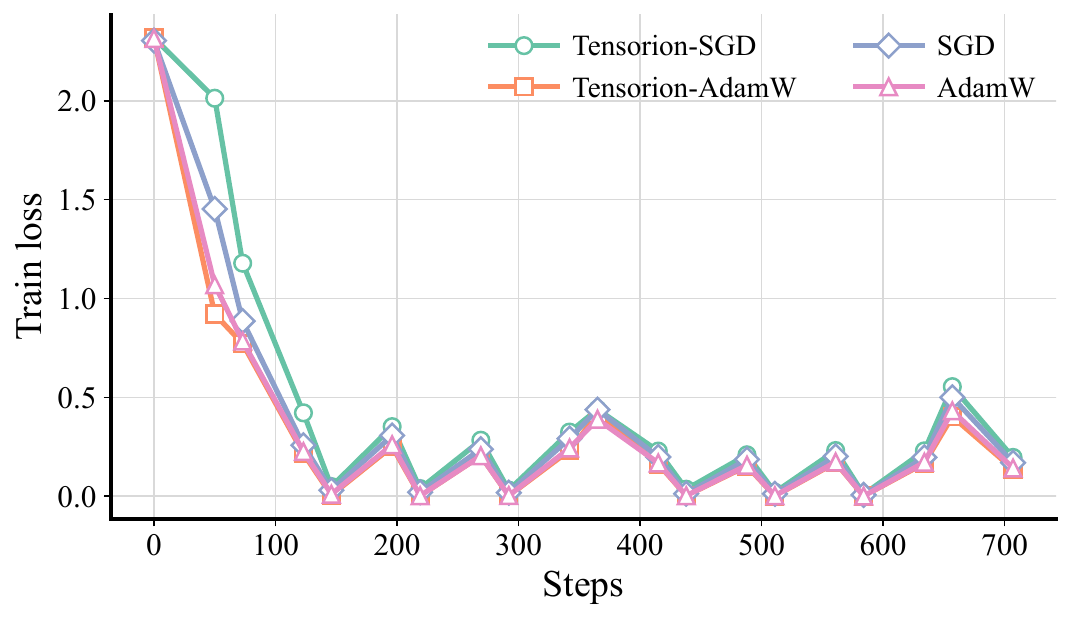}
        \caption{Swin Small Train Loss}
        \label{fig:swin_small_loss}
    \end{subfigure}
    \begin{subfigure}{0.49\linewidth}
        \centering
        \includegraphics[width=\linewidth]{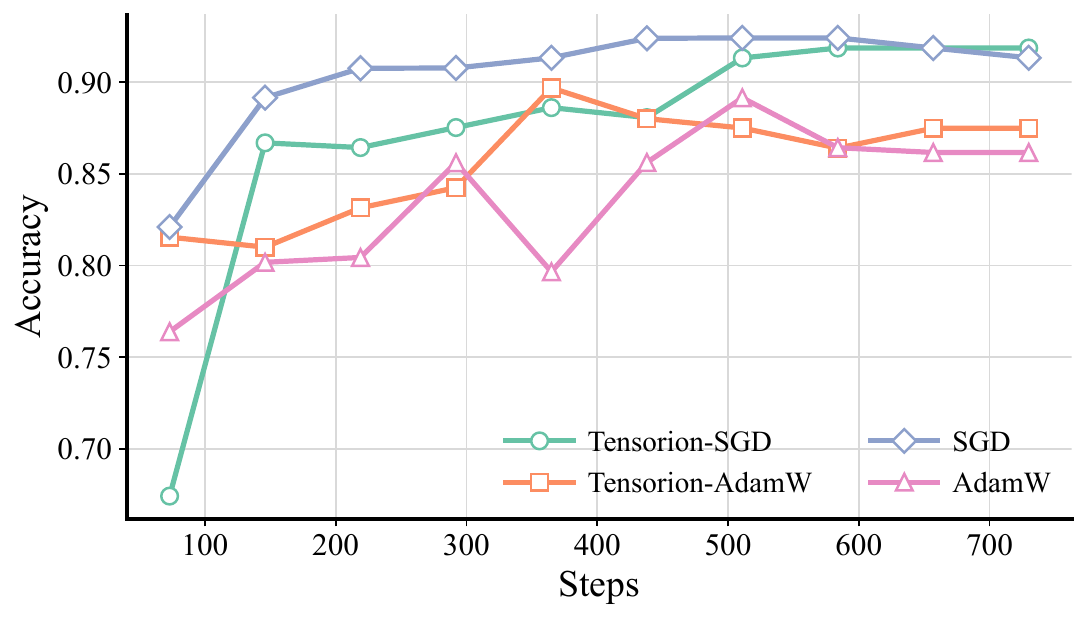}
        \caption{Swin Small Accuracy}
        \label{fig:swin_small_acc}
    \end{subfigure}
    
    \begin{subfigure}{0.49\linewidth}
        \centering
        \includegraphics[width=\linewidth]{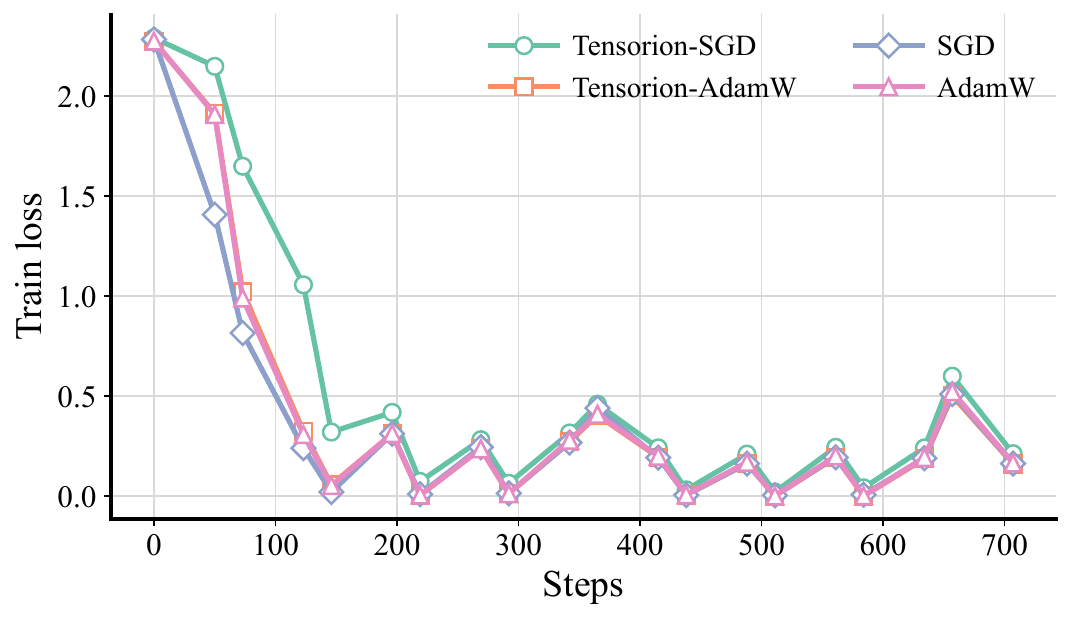}
        \caption{Swin Base Train Loss}
        \label{fig:swin_base_loss}
    \end{subfigure}
    \begin{subfigure}{0.49\linewidth}
        \centering
        \includegraphics[width=\linewidth]{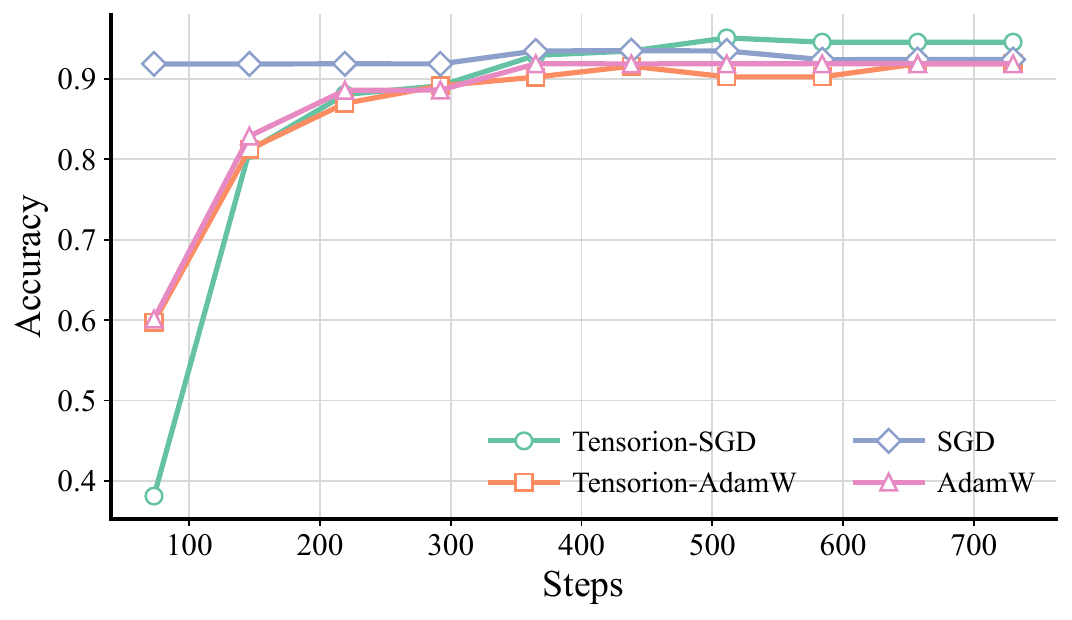}
        \caption{Swin Base Accuracy}
        \label{fig:swin_base_acc}
    \end{subfigure}

    \caption{Test accuracy and training loss across all Swin models.}
    \label{fig:all_model_scores_swin}
\end{figure}

\begin{figure}[h!tp]
    \centering

    \begin{subfigure}{0.49\linewidth}
        \centering
        \includegraphics[width=\linewidth]{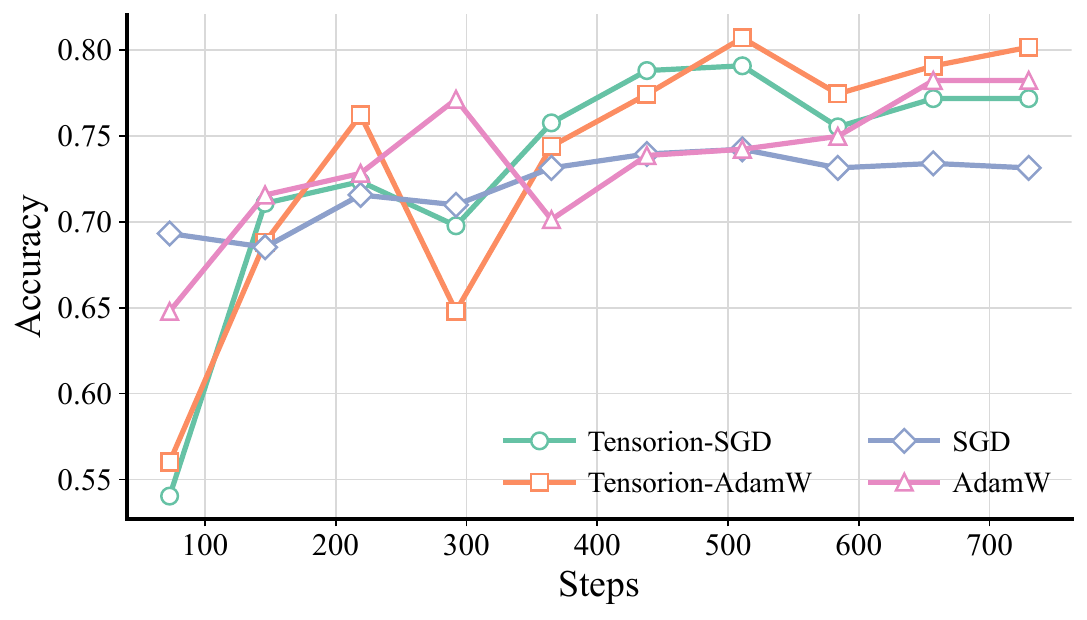}
        \caption{ViT Small Accuracy}
        \label{fig:vit_small_acc}
    \end{subfigure}
    \begin{subfigure}{0.49\linewidth}
        \centering
        \includegraphics[width=\linewidth]{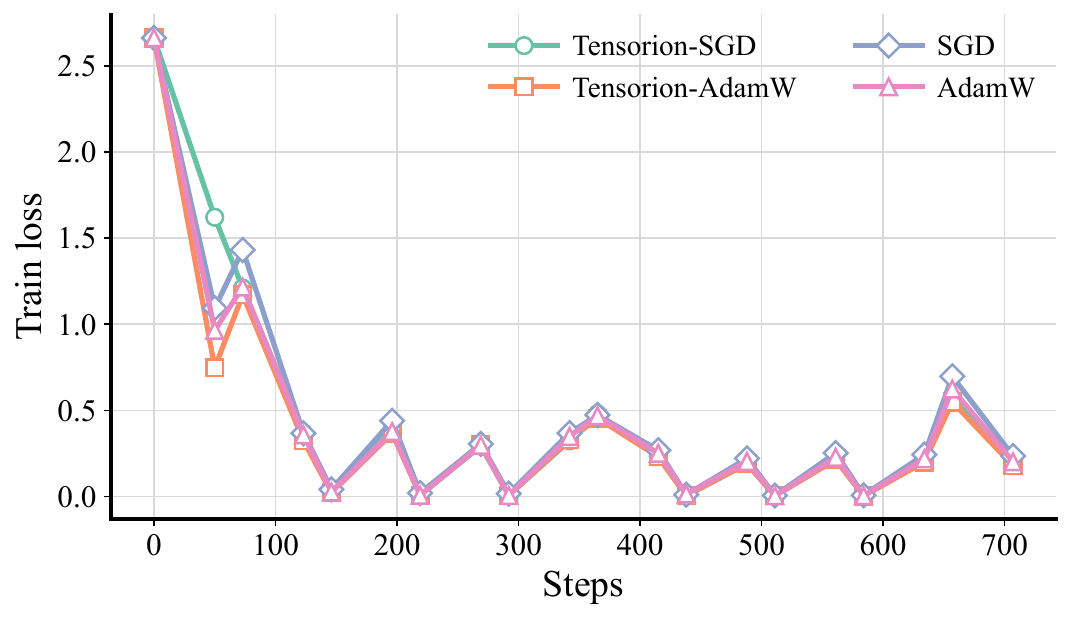}
        \caption{ViT Small Train Loss}
        \label{fig:vit_small_loss}
    \end{subfigure}

    \begin{subfigure}{0.49\linewidth}
        \centering
        \includegraphics[width=\linewidth]{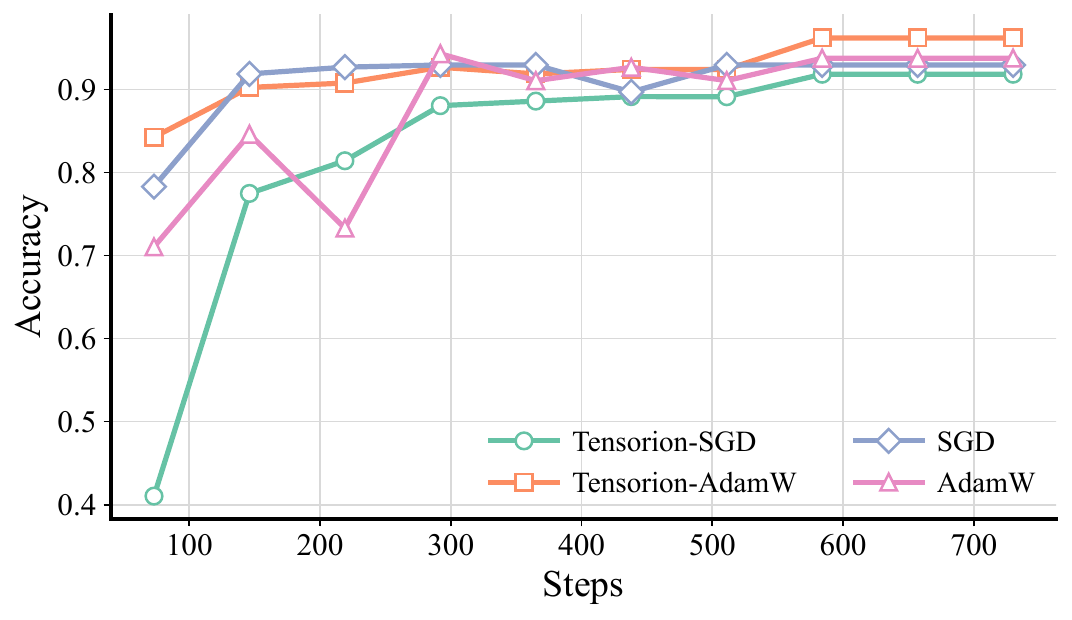}
        \caption{ViT Base Accuracy}
        \label{fig:vit_base_acc}
    \end{subfigure}
    \begin{subfigure}{0.49\linewidth}
        \centering
        \includegraphics[width=\linewidth]{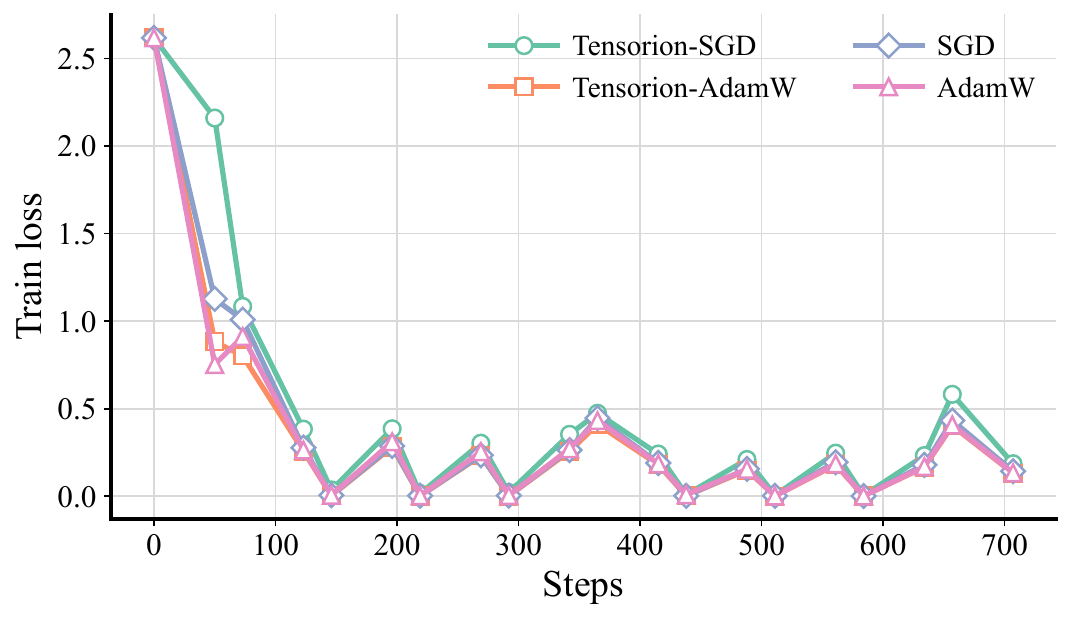}
        \caption{ViT Base Train Loss}
        \label{fig:vit_base_loss}
    \end{subfigure}

    \begin{subfigure}{0.49\linewidth}
        \centering
        \includegraphics[width=\linewidth]{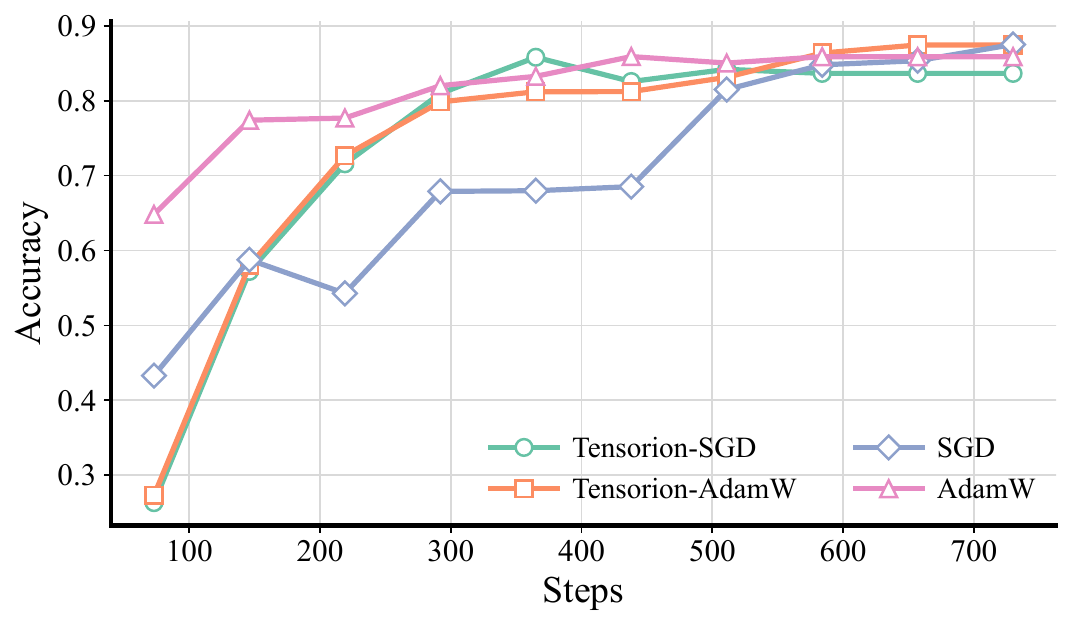}
        \caption{ViT Medium Accuracy}
        \label{fig:vit_medium_acc}
    \end{subfigure}
    \begin{subfigure}{0.49\linewidth}
        \centering
        \includegraphics[width=\linewidth]{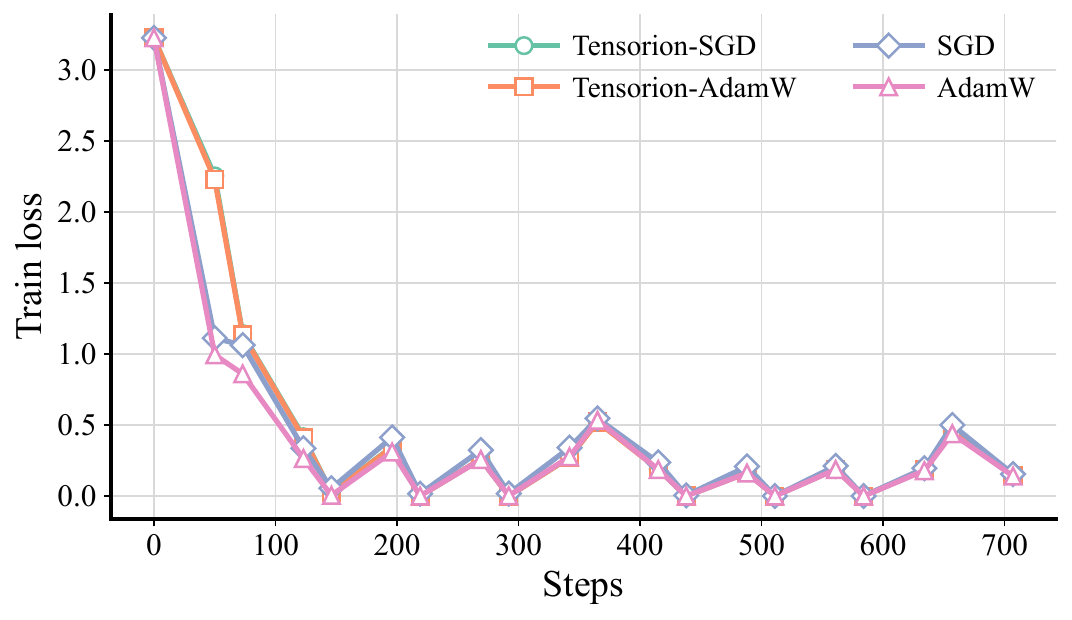}
        \caption{ViT Medium Train Loss}
        \label{fig:vit_medium_loss}
    \end{subfigure}

    \caption{Test accuracy and training loss across ViT models.}
    \label{fig:all_model_scores_vit}
\end{figure}

\subsection{Beyond CNNs: Laplacian Spectrum}
\label{subsec:laplacian-spectrum}

\paragraph{Tensor trains (TT).}
Tensor Train (TT) decomposition is a low-rank factorization for high-order tensors that represents a $d$-way tensor $x \in \mathbb{R}^{n_1 \times \cdots \times n_d}$ as a product of \emph{TT-cores}:
\begin{equation}
x(i_1,\dots,i_d)
=
G^{(1)}(i_1)\, G^{(2)}(i_2)\cdots G^{(d)}(i_d),
\label{eq:tt-decomp}
\end{equation}
where each core $G^{(k)}(i_k) \in \mathbb{R}^{r_{k-1}\times r_k}$ is a matrix indexed by $i_k \in \{1,\dots,n_k\}$, with boundary ranks $r_0=r_d=1$. The tuple $(r_1,\dots,r_{d-1})$ are the TT-ranks; small TT-ranks yield a compressed representation that enables scalable optimization in very high dimensions.

\paragraph{Discrete Laplacian on $[0,1]^d$.}
Consider a uniform tensor-product grid on $[0,1]^d$ with $n$ points per dimension and step size $h = \frac{1}{n+1}$ (Dirichlet boundary conditions). Let $L \in \mathbb{R}^{n\times n}$ be the standard 1D second-difference matrix
\begin{equation}
L = \frac{1}{h^2}\,\mathrm{tridiag}(1,-2,1),
\label{eq:1d-laplacian}
\end{equation}
where $\mathrm{tridiag}(1,-2,1)$ denotes tridiagonal matrix.

The $d$-dimensional discrete Laplacian $\Delta_d \in \mathbb{R}^{n^d \times n^d}$ on the tensor-product grid is given by the Kronecker sum
\begin{equation}
\Delta_d
=
\sum_{k=1}^{d}
I^{\otimes (k-1)} \otimes L \otimes I^{\otimes (d-k)},
\label{eq:kronecker-sum-laplacian}
\end{equation}
where $I \in \mathbb{R}^{n\times n}$ is the identity and $\otimes$ denotes the Kronecker product. We treat $\Delta_d$ as a linear operator acting on an order-$d$ tensor $x \in \mathbb{R}^{n\times \cdots \times n}$ via vectorization.

\paragraph{TT-operator form.}
Each term $I^{\otimes (k-1)} \otimes L \otimes I^{\otimes (d-k)}$ in \eqref{eq:kronecker-sum-laplacian} is a rank-1 Kronecker product of $d$ factors and thus admits an exact TT-operator representation with TT-rank $1$. Consequently, the full Laplacian $\Delta_d$, being a sum of $d$ such Kronecker terms, admits a compact TT-operator representation with small TT-ranks (bounded by a constant independent of $n^d$), which allows for efficient application $\Delta_d x$ and inner products $\langle x, \Delta_d x\rangle$ in TT format.

\paragraph{Low-rank structure of eigenvectors.}
Our goal is to compute the low end of the spectrum of $\Delta_d$ in a \emph{compressed} form. We restrict the search space to eigenvectors representable with small TT-ranks, i.e., $x$ is parameterized by TT-cores with ranks bounded by a prescribed $r$. This reflects the practical regime in high-dimensional PDE discretizations where one seeks approximate eigenmodes within a low-rank manifold.

\paragraph{Rayleigh quotient optimization.}
Let $A := -\Delta_d$ be the discrete positive semidefinite operator.
The smallest eigenpair of $A$ can be obtained by minimizing the Rayleigh quotient
\begin{equation}
\min_{x \neq 0} \ \mathcal{R}(x)
\;=\;
\min_{x \neq 0} \ \frac{\langle x, A x\rangle}{\langle x, x\rangle},
\label{eq:rayleigh-quotient}
\end{equation}
In our experiments, we optimize \eqref{eq:rayleigh-quotient} over the TT-parameterization of $x$ with bounded TT-ranks, comparing Tensorion to gradient descent baselines on this non-deep, structured optimization task.

\begin{figure*}[h]
    \centering

    \begin{subfigure}[t]{0.49\linewidth}
        \centering
        \includegraphics[width=\linewidth]{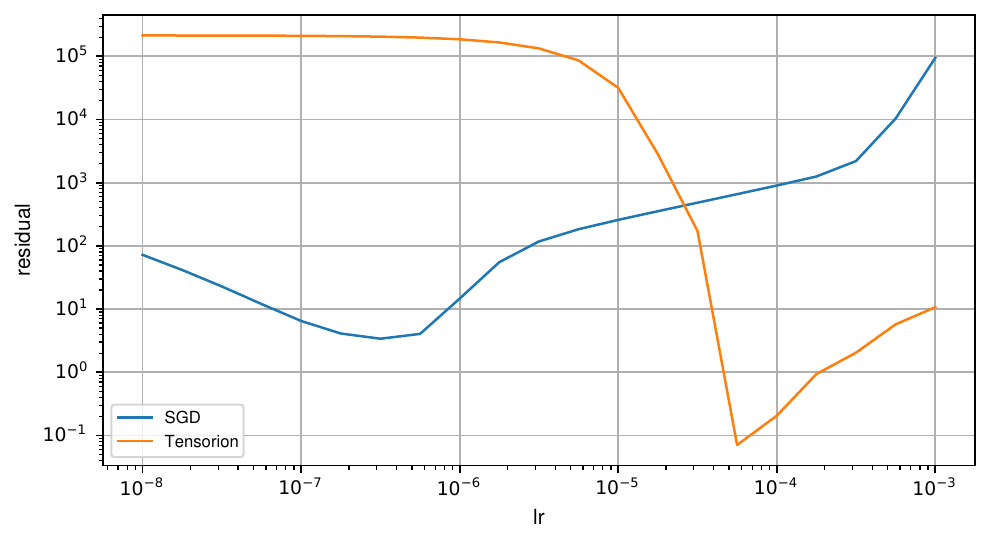}
        \caption{Residual vs learning rate}
        \label{fig:laplace-residual-vs-lr}
    \end{subfigure}\hfill
    \begin{subfigure}[t]{0.49\linewidth}
        \centering
        \includegraphics[width=\linewidth]{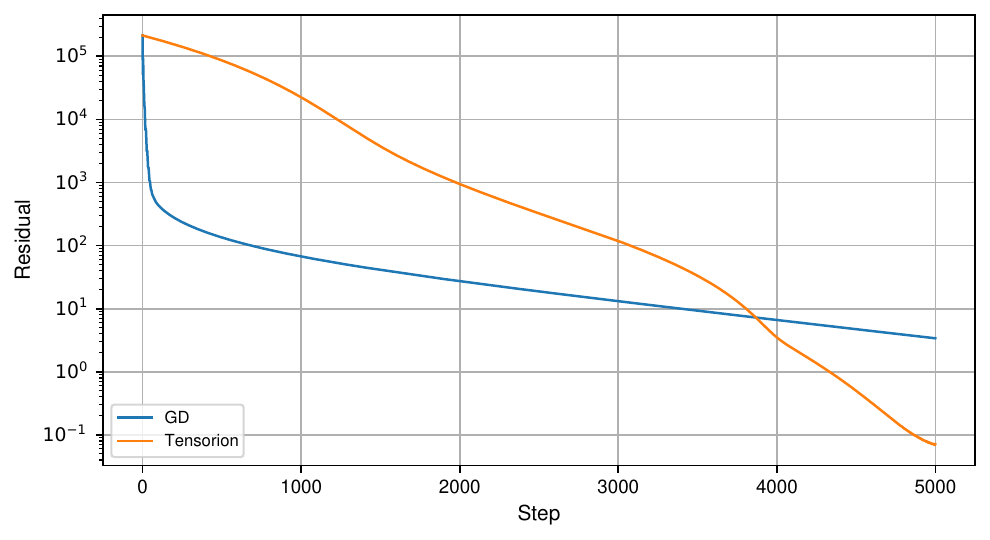}
        \caption{Residual vs step, optimal learning rate}
        \label{fig:laplace-residual-vs-step}
    \end{subfigure}

    \caption{Minimization of the Rayleigh quotient \eqref{eq:rayleigh-quotient} for the smallest eigenpair of \(A=-\Delta_d\) on a \(d\)-dimensional tensor-product grid, with the eigenvector \(x\) parameterized in TT format with bounded TT-ranks. Left: final functional-residual as a function of the learning rate. Right: functional-residual versus optimization step at the best-tuned learning rate for each method.}
    \label{fig:rayleigh-minimizing}
\end{figure*}

\section{Relaxation Tightness}\label{sec:appendix_relaxation_tightness}

To empirically assess the approximations underlying the practical Tensorion optimizer, we consider two complementary questions. First, we examine the tightness of the Tensorion norm relaxation relative to the tensor spectral norm along the LMO directions encountered during optimization. Second, we study the additional approximation introduced by the practical offline variant of Tensorion, which replaces the exact per-iteration search for an unfolding with the largest nuclear norm, used by online Tensorion, with the shape-based heuristic in Eq.~\eqref{eq:tau_opt_unfolding}. Both analyses are performed using optimization trajectories obtained by training ResNet-18 on CIFAR-10 with Tensorion-SGD. Along these trajectories, we collect the momentum tensors $M_{k,l}$ for each tensor-valued layer $l$ at optimization step $k$ and use the same collection of tensors in both experiments. This setup allows us to separately evaluate the quality of the norm relaxation and the accuracy of the offline approximation on tensors encountered during the actual optimization process.

\subsection{Tightness of the Tensorion Norm Relaxation}

We first examine the tightness of the relaxation introduced by replacing the tensor spectral norm with the Tensorion norm. Recall from Proposition~\ref{proposition:tensorion_and_spectral_norm_ineqs} that, for any tensor $X$,
\begin{equation}
    \| X \|_\sigma \leq \| X \|_\Sigma.
\end{equation}
For each collected momentum tensor $M_{k,l}$, we compute the online Tensorion update by solving the LMO in Eq.~\eqref{eq:tensorion_norm_LMO}. In particular, online Tensorion selects
\begin{equation}
    \tau_{k,l}^{\mathrm{online}}
    \in
    \arg\max_{\tau \in \Tau}
    \| (M_{k,l})_{(\tau)} \|_*,
\end{equation}
and uses the corresponding solution $X_{k,l}^{\Sigma}$ of the Tensorion LMO. Since $\| X_{k,l}^{\Sigma} \|_\Sigma = 1$, the inequality above implies
\begin{equation}
    \| X_{k,l}^{\Sigma} \|_\sigma \leq 1.
\end{equation}
We therefore use $\| X_{k,l}^{\Sigma} \|_\sigma$ as an empirical measure of the tightness of the Tensorion norm relaxation along the update directions encountered during optimization. Values close to one indicate that, along these directions, the boundary of the Tensorion unit ball is close to the boundary of the tensor spectral norm unit ball.

Figure~\ref{fig:relaxation-tightness} reports the tensor spectral norm $\| X_{k,l}^{\Sigma} \|_\sigma$ along the optimization trajectory, averaged across tensor-valued layers. The values remain consistently close to one, indicating that the Tensorion norm provides a tight relaxation of the tensor spectral norm for the update directions observed during training. Importantly, this experiment does not require solving the LMO over the tensor spectral norm ball, which is computationally intractable in general. Instead, it evaluates how close the exact online Tensorion LMO solutions are to the boundary of that ball.

\begin{figure}[t!]
    \centering
    \includegraphics[width=0.65\linewidth]{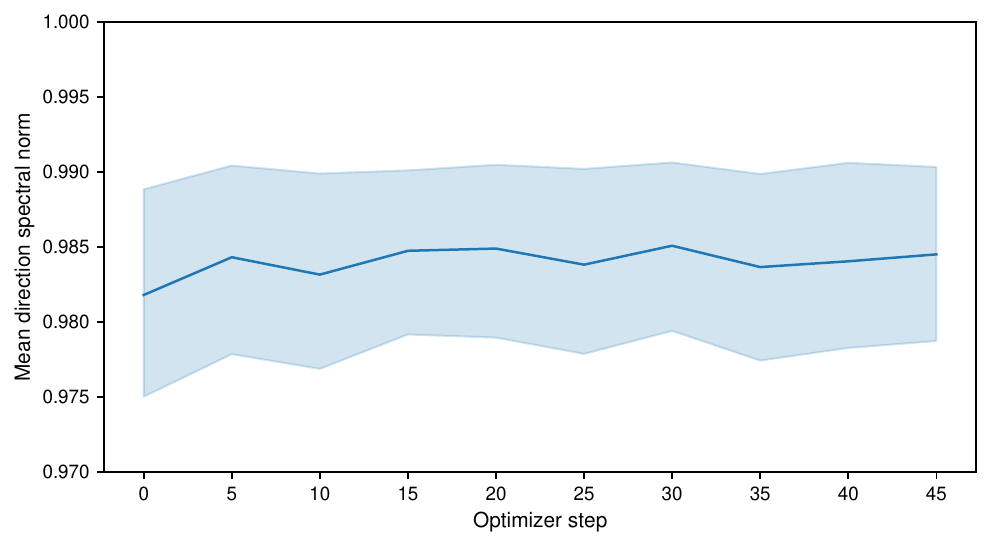}
    \caption{Tensor spectral norm $\| X_{k,l}^{\Sigma} \|_\sigma$ of the online Tensorion LMO solutions along the optimization trajectory of ResNet-18 on CIFAR-10. The values are averaged across tensor-valued layers.}
    \label{fig:relaxation-tightness}
\end{figure}

\subsection{Online vs.\ Offline Tensorion}

We next evaluate how closely the offline unfolding heuristic approximates the exact online Tensorion LMO. For each collected momentum tensor $M_{k,l}$, the online strategy selects
\begin{equation}
    \tau_{k,l}^{\mathrm{online}}
    \in
    \arg\max_{\tau \in \Tau}
    \| (M_{k,l})_{(\tau)} \|_*,
    \label{eq:online_tau_selection}
\end{equation}
whereas the offline strategy uses the unfolding $\tau_l^{\mathrm{opt}}$ selected once for layer $l$ according to Eq.~\eqref{eq:tau_opt_unfolding}. Thus, the online unfolding may change throughout optimization, while the offline unfolding depends only on the tensor shape of the corresponding layer.

We first compare the nuclear norm values attained by the two selection strategies. For the online strategy, this value is $\max_{\tau \in \Tau} \| (M_{k,l})_{(\tau)} \|_*$ whereas for the offline strategy it is $\| (M_{k,l})_{(\tau_l^{\mathrm{opt}})} \|_*$.

Figure~\ref{fig:online-offline-nuclear-norm} shows kernel density estimates of these quantities over all collected momentum tensors $M_{k,l}$. The two distributions almost completely overlap, indicating that the offline heuristic typically selects an unfolding whose nuclear norm is very close to the maximum obtained by the exact online search.

To quantify this discrepancy for each individual momentum tensor, we consider the relative nuclear norm gap
\begin{equation}
    \operatorname{gap}(M_{k,l})
    =
    \frac{
        \| (M_{k,l})_{(\tau_{k,l}^{\mathrm{online}})} \|_*
        -
        \| (M_{k,l})_{(\tau_l^{\mathrm{opt}})} \|_*
    }{
        \| (M_{k,l})_{(\tau_{k,l}^{\mathrm{online}})} \|_*
    }.
    \label{eq:online_offline_relative_gap}
\end{equation}
Since the maximal unfolding nuclear norm equals the optimal value of the Tensorion LMO, this quantity can also be interpreted as the relative LMO objective suboptimality introduced by the offline unfolding heuristic.

By construction, $\operatorname{gap}(M_{k,l}) \geq 0$, and it equals zero whenever the offline-selected unfolding attains the same nuclear norm as the online optimum. Across all collected momentum tensors, the mean relative gap is $0.0039$, while the maximum observed gap is $0.0732$. Hence, the nuclear norm attained by the offline heuristic is, on average, only $0.39\%$ below the online optimum, while the largest observed discrepancy is $7.32\%$. Together with the nearly coinciding empirical distributions, these results show that the shape-based offline heuristic typically achieves a nuclear norm value very close to that obtained by the exact per-iteration unfolding search of online Tensorion.

\begin{figure}[t!]
    \centering
    \includegraphics[width=0.65\linewidth]{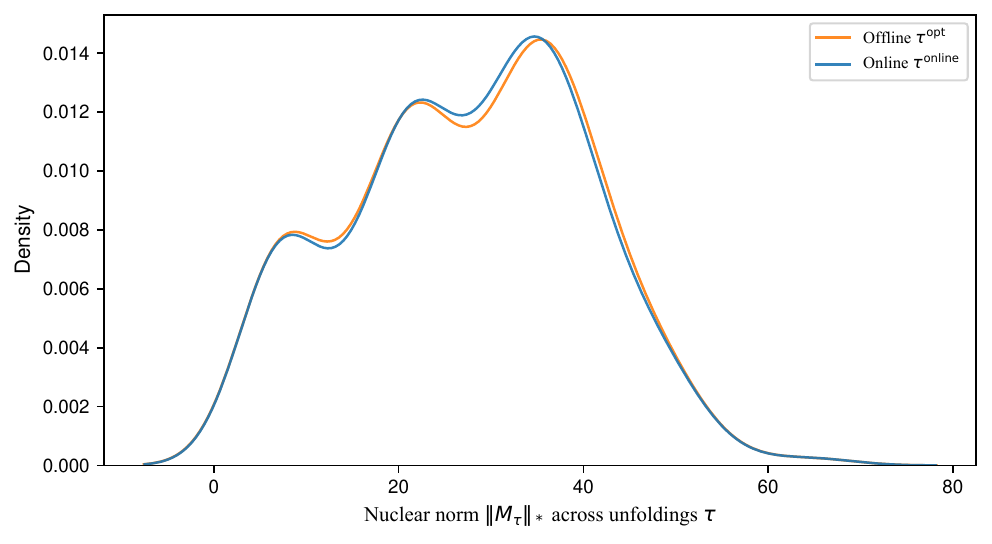}
    \caption{
        Kernel density estimates of the nuclear norms obtained by online and offline Tensorion over momentum tensors collected along the optimization trajectory of ResNet-18 on CIFAR-10. Online Tensorion uses
        $\max_{\tau \in \Tau} \| (M_{k,l})_{(\tau)} \|_*$,
        whereas offline Tensorion uses
        $\| (M_{k,l})_{(\tau_l^{\mathrm{opt}})} \|_*$.
    }
    \label{fig:online-offline-nuclear-norm}
\end{figure}

\section{Experimental Details}\label{sec:appendix_experimental_details}
\subsection{Tau ablation}\label{subsec:appendix_tau_ablation}

We enumerate all proper non-empty subsets $\tau \subsetneq \{1,2,3,4\}$.
Since unfolding with $\tau$ and with its complement $\tau^c$ produces matrices that are transposes of each other, and the nuclear norm is invariant to this operation, $\|X_{(\tau)}\|_* = \|X_{(\tau^c)}\|_*$, it suffices to consider $\tau$ up to the equivalence $\tau \sim \tau^c$, yielding $(2^4-2)/2=7$ distinct cases for $4$-dimensional convolution kernel. 

The convolutional kernel has shape $(C_{\text{out}}, C_{\text{in}}, k_h, k_w)$. In ResNet architectures, it typically holds that $C_{\text{out}}~\ge~C_{\text{in}}~\gg~k_h, k_w$.

In case of static $\tau$ selection, the unfolding is fixed across all layers, while the offline strategy selects $\tau^{\text{opt}}$ for each layer individually. To decouple the effect of $\tau$ selection from numerical approximation artifacts, we replace the Newton--Schulz routine (line $7$ in Algorithm~\ref{alg:one_step_of_tensorion}) with an exact SVD throughout this ablation. AdamW was used for layers with weight tensors of order $d \le 2$.

We train ResNet-18~\citep{he2016deep} on CIFAR-10~\citep{Krizhevsky09learningmultiple} for a fixed budget of 25 epochs (no early stopping). Across all optimizers, we keep the batch size, weight decay, and image preprocessing identical. We do not use an LR scheduler; instead, the learning rate is tuned \emph{independently for each method} via a grid search (see Table~\ref{tab:tau-ablation-lr}). 

\section{Hyperparameters}\label{sec:appendix_hyperparameters}
The optimal hyperparameter configurations reported  throughout the experiments are summarized in Table~\ref{tab:hyperparameters} and Table~\ref{tab:tau-ablation-lr}.

\begin{table}[h!t]
\centering
\caption{Hyperparameter configurations used across datasets and optimization setups.}
\label{tab:hyperparameters}
\begin{tabular}{lcccc}
\toprule
Setup / Method & SGD & AdamW & Tensorion & Tensorion\_AdamW \\
\midrule
CIFAR10 Scheduler  & \texttt{Multistep} & \texttt{Multistep} & \texttt{Onecycle} & \texttt{Onecycle} \\
CIFAR10 LR         & $0.075$ & $0.0001$ & $0.001$ & $0.075$ \\
CIFAR10 WD         & $0.0005$ & $0.0001$ & $0.0001$ & $0.0005$ \\
\midrule
CIFAR100 Scheduler & \texttt{Multistep} & \texttt{Multistep} & \texttt{Onecycle} & \texttt{Onecycle} \\
CIFAR100 LR        & $0.05$ & $0.00025$ & $0.0075$ & $0.00005$ \\
CIFAR100 WD        & $0.0005$ & $0.0005$ & $0.0005$ & $0.0005$ \\
\midrule
Tiny Scheduler     & \texttt{Onecycle} & \texttt{Onecycle} & \texttt{Onecycle} & -- \\
Tiny LR            & $0.01$ & $0.0005$ & $0.075$ & -- \\
Tiny WD            & $0.0005$ & $0.0001$ & $0.0005$ & -- \\
\midrule
Imagenette Scheduler     & \texttt{Cosine} & \texttt{Cosine} & \texttt{Cosine} & \texttt{Cosine} \\
Imagenette LR            & $0.005$ & $0.0001$ & $0.0005$ & $0.0001$ \\
Imagenette WD            & $0.0005$ & $0.0005$ & $0.0005$ & $0.0005$ \\
\bottomrule
\end{tabular}
\end{table}
\begin{table}[h]
\centering
\caption{Learning rates used in the ablation on the unfolding set $\tau$ for ResNet-18 on CIFAR-10 after 25 epochs. The learning rate is tuned per method via grid search.}
\label{tab:tau-ablation-lr}
\small
\setlength{\tabcolsep}{4pt}
\renewcommand{\arraystretch}{1.15}
\resizebox{\textwidth}{!}{%
\begin{tabular}{lccccccc c c cc}
\toprule
& \multicolumn{7}{c}{\textbf{Tensorion (static $\tau$)}} 
& \multicolumn{1}{c}{\textbf{Offline}} 
& \multicolumn{1}{c}{\textbf{Online}} 
& \multicolumn{2}{c}{\textbf{Baselines}} \\
\cmidrule(lr){2-8}
\cmidrule(lr){9-9}
\cmidrule(lr){10-10}
\cmidrule(lr){11-12}
\textbf{Hyperparameter} 
& $\{1\}$ 
& $\{2\}$ 
& $\{3\}$ 
& $\{4\}$ 
& $\{1,2\}$ 
& $\{1,3\}$ 
& $\{1,4\}$ 
& $\tau^{\mathrm{opt}}$ 
& online $\tau$ 
& SGD+M. 
& AdamW \\
\midrule
Learning rate
& $0.01$
& $0.001$
& $0.001$
& $0.005$
& $0.005$
& $0.001$
& $0.005$
& $0.01$
& $0.005$
& $0.05$
& $0.0005$ \\
\bottomrule
\end{tabular}%
}
\end{table}

\section{Compute resources}\label{sec:appendix_comp_res}

All experiments were conducted on a shared high-performance computing infrastructure. 
The experiments reported in Section~\ref{subsec:ablation-tau} required approximately 300 GPU-hours on NVIDIA V100-SXM2-32GB accelerators. 
All remaining experiments, including baseline training, ablation studies, and final model evaluations, consumed approximately 2{,}000 GPU-hours on NVIDIA A100-SXM-80GB accelerators. 

In total, the computational budget for the reported results amounts to roughly 2{,}300 GPU-hours. 

Hyperparameter search, preliminary runs, and failed attempts are included in these estimates. We confirm that the paper reports the full set of experiments conducted within this budget, and no selective reporting based on compute availability was applied. 

\end{document}